\documentclass[colorlinks]{article}[12pt]
\pdfoutput=1  
\PassOptionsToPackage{numbers}{natbib}
\usepackage[preprint]{neurips_2020}

\def\forarxiv #1{\footnote{#1}}

\usepackage{parskip}

\usepackage{amsthm,amsmath,bbm,amsfonts,amssymb,commath}
\usepackage{dsfont} 
\usepackage{mathtools}
\mathtoolsset{showonlyrefs=false}

\usepackage{xspace}
\usepackage[T1]{fontenc}
\usepackage[utf8]{inputenc}
\usepackage[usenames,dvipsnames]{xcolor} 

\usepackage[hyperindex,breaklinks]{hyperref}
\usepackage{url}
\usepackage{cleveref} 

\usepackage[usenames,dvipsnames]{xcolor} 

\usepackage{wrapfig}
\usepackage[export]{adjustbox}
\usepackage{algorithm,algorithmic}
\usepackage{algorithmic}

\usepackage{graphicx,tabularx}
\usepackage{multirow,hhline}
\usepackage{booktabs}

\usepackage{capt-of}

\allowdisplaybreaks 

\usepackage{pbox} 

\usepackage[export]{adjustbox}

\usepackage{pifont}
\newcommand{\blue}[1]{{\color[rgb]{.3,.5,1}#1}}

\usepackage{framed}
\usepackage[most]{tcolorbox}
\definecolor{kjgray}{rgb}{.7,.7,.7}

\newtheoremstyle{kjstyle}
{1ex} 
{\topsep} 
{\itshape} 
{} 
{\bfseries} 
{.} 
{.5em} 
{} 

\tcbset{kjboxstyle/.style={title={},breakable,colback=white,enhanced jigsaw,boxrule=1.3pt,sharp corners,colframe=kjgray,boxsep=0pt,coltitle={black},attach title to upper={},left=.8ex,bottom=.4em}}

\usepackage{mdframed}
\usepackage{lipsum}
\definecolor{kjgray}{rgb}{.7,.7,.7}
\makeatletter
\newcommand{\kjunbox}[1]{\vspace{.4ex}
  \begin{mdframed}[linecolor=kjgray,innertopmargin=1.3ex,innerleftmargin=.4em,innerrightmargin=.4em,linewidth=1.3pt]
    #1
  \end{mdframed}%
}


\makeatletter
\renewcommand{\paragraph}{%
  \@startsection{paragraph}{4}%
  {\z@}{0.50ex \@plus 1ex \@minus .2ex}{-1em}%
  {\normalfont\normalsize\bfseries}%
}
\makeatother

\usepackage[customcolors,shade]{hf-tikz} 

\def\ddefloop#1{\ifx\ddefloop#1\else\ddef{#1}\expandafter\ddefloop\fi}
\def\ddef#1{\expandafter\def\csname c#1\endcsname{\ensuremath{\mathcal{#1}}}}
\ddefloop ABCDEFGHIJKLMNOPQRSTUVWXYZ\ddefloop
\def\ddef#1{\expandafter\def\csname b#1\endcsname{\ensuremath{{\mathbf{#1}}}}}
\ddefloop ABCDEFGHIJKLMNOPQRSUVWXYZabcdeghijklmnopqrtsuvwxyz\ddefloop  
\def\ddef#1{\expandafter\def\csname h#1\endcsname{\ensuremath{\widehat{#1}}}}
\ddefloop ABCDEFGHIJKLMNOPQRSTUVWXYZabcdefghijklmnopqrsuvwxyz\ddefloop 
\def\ddef#1{\expandafter\def\csname hc#1\endcsname{\ensuremath{\widehat{\mathcal{#1}}}}}
\ddefloop ABCDEFGHIJKLMNOPQRSTUVWXYZ\ddefloop
\def\ddef#1{\expandafter\def\csname t#1\endcsname{\ensuremath{\widetilde{#1}}}}
\ddefloop ABCDEFGHIJKLMNOPQRSTUVWXYZabcdefghijklmnopqrtsuvwxyz\ddefloop
\def\ddef#1{\expandafter\def\csname r#1\endcsname{\ensuremath{\mathring{#1}}}}
\ddefloop ABCDEFGHIJKLMNOPQRSTUVWXYZabcdefghijklmnopqrtsuvwxyz\ddefloop
\def\ddef#1{\expandafter\def\csname tc#1\endcsname{\ensuremath{\widetilde{\mathcal{#1}}}}}
\ddefloop ABCDEFGHIJKLMNOPQRSTUVWXYZ\ddefloop

\def\SS{{\mathbb{S}}}

\DeclareMathOperator*{\argmax}{arg~max}
\DeclareMathOperator*{\argmin}{arg~min}
\DeclareMathOperator{\EE}{\mathbb{E}}

\DeclareMathOperator{\PP}{\mathbb{P}}

\DeclareMathOperator{\one}{\mathds{1}}
\DeclareMathOperator{\Reg}{{\text{Reg}}}

\newcommand{\lapp}{\mathop{}\!\lessapprox}

\def\RR{{\mathbb{R}}}
\def\NN{{\mathbb{N}}}

\newcommand{\ubrace}[2]{{\underbrace{#1}_{\textstyle #2}}}
\newcommand{\sr}[2]{ {\stackrel{#1}{#2}} }

\newcommand{\fr}[2]{ { \frac{#1}{#2} }}
\def\lt{\left}
\def\rt{\right}
\def\suchthat{\ensuremath{\mbox{ s.t. }}}

\def\eps{\ensuremath{\epsilon}\xspace}

\def\rarrow{\ensuremath{\rightarrow}\xspace}
\def\sig{\ensuremath{\sigma}\xspace}

\def\eps{\ensuremath{\epsilon}\xspace}
\def\diag{\ensuremath{\mbox{diag}}\xspace}

\def\dt{{\ensuremath{\delta}\xspace} }
\def\sm{{\ensuremath{\setminus}\xspace} }

\def\gam{{\ensuremath{\gamma}\xspace} }

\makeatletter
\newcommand{\vast}{\bBigg@{3}}
\newcommand{\Vast}{\bBigg@{4}}
\makeatother

\def\la{{\langle}}
\def\ra{{\rangle}}

\def\lam{\ensuremath{\lambda}}

\def\tilmu{{\ensuremath{\tilde{\mu}}\xspace} }

\def\cC{\ensuremath{\mathcal{C}}\xspace}
\def\cD{\ensuremath{\mathcal{D}}\xspace}

\def\KL{\ensuremath{\mathsf{KL}}}

\newcommand{\wtil}[1]{ {\ensuremath{\widetilde{#1}}} }

\def\vec{\text{vec}}

\newcommand{\wbar}[1]{ {\ensuremath{\overline{#1}}} }

\def\lam{{\ensuremath{\lambda}\xspace} }

\def\vec{\ensuremath{\text{vec}}}

\def\Lam{{{{{\Lambda}}}}}

\def\cF{{\ensuremath{\mathcal{F}}}}

\def\cO{{\ensuremath{\mathcal{O}}}}

\def\poly{\operatorname{poly}}

\def\cH{{\ensuremath{\mathcal{H}}}}

\let\SS\undefined

\def\suchthat{\text{ s.t. }}
\def\cd{\cdot}

\def\hf{\ensuremath{\widehat{f}}}

\def\tila{\ensuremath{\widetilde{a}}}
\def\tilmu{\ensuremath{\widetilde{\mu}}}
\def\tilf{\ensuremath{\widetilde{f}}}
\def\tilcF{\ensuremath{\widetilde{\cF}}}
\def\barcF{\ensuremath{\overline{\cF}}}
\def\barf{\ensuremath{\overline{f}}}
\def\bara{\ensuremath{\overline{a}}}
\def\barmu{\ensuremath{\overline{\mu}}}

\def\IC{{\ensuremath{\normalfont{\text{IC}}}}}

\def\Reg{{\ensuremath{\normalfont{\text{Reg}}}}}
\newcommand{\barB}{{\ensuremath{\wbar B}}}
\newcommand{\barEx}{{\ensuremath{\wbar{\text{\normalfont Ex}}}}}

\def\rcF{\ensuremath{\mathring{\cF}}}
\def\rbeta{\ensuremath{\mathring{\beta}}}

\def\Ex{{\ensuremath{{\normalfont{\text{Ex}}}}}}
\def\Cf{{\ensuremath{{\normalfont{\text{Cf}}}}}}
\def\Fs{{\normalfont{\ensuremath{\text{Fs}}}}}
\def\Fb{{\normalfont{\ensuremath{\text{Fb}}}}}

\def\ralpha{\ensuremath{{\mathring{\alpha}}}}
\def\rz{\ensuremath{\mathring{z}}}

\newcommand{\crop}{{\ensuremath{\textsc{CROP}}}\xspace}

\def\plabel#1{{\phantomsection\label{#1}}}

\newcommand\order[1]{O\del{{#1}}}

\def\exploit{{\normalfont\texttt{Exploit}}\xspace}
\def\conflict{{\normalfont\texttt{Conflict}}\xspace}
\def\feasible{{\normalfont\texttt{Feasible}}\xspace}
\def\fallback{{\normalfont\texttt{Fallback}}\xspace}

\newcolumntype{L}[1]{>{\raggedright\let\newline\\\arraybackslash\hspace{0pt}}m{#1}}
\newcolumntype{C}[1]{>{\centering\let\newline\\\arraybackslash\hspace{0pt}}m{#1}}
\newcolumntype{R}[1]{>{\raggedleft\let\newline\\\arraybackslash\hspace{0pt}}m{#1}}

\def\code{\text{code}}

 \newcommand\bigsubseteq[1][1.19]{%
   \mathrel{\vcenter{\hbox{\scalebox{#1}{$\subseteq$}}}}}
\newcommand\zo{\text{(Z1)}}
\newcommand\zoa{\text{(Z1-a)}}
\newcommand\zob{\text{(Z1-b)}}
\newcommand\zoc{\text{(Z1-c)}}
\newcommand\zt{\text{(Z2)}}
\newcommand\zta{\text{(Z2-a)}}
\newcommand\ztb{\text{(Z2-b)}}
\newcommand\ztc{\text{(Z2-c)}}
\def\bec{\ensuremath{\because\mathop{}}\ }

\RequirePackage[OT1]{fontenc}

\newcommand\myshade{85}
\colorlet{mylinkcolor}{MidnightBlue}
\hypersetup{colorlinks,
  linkcolor=mylinkcolor,
  citecolor=mylinkcolor!\myshade!black,
  urlcolor=mylinkcolor,
}

\renewcommand{\cite}{\citep}

\usepackage{lmodern}
\usepackage{booktabs}       
\usepackage{nicefrac}       
\usepackage{microtype}      

\usepackage{anyfontsize}
\usepackage{newtxtext}  
\let\vec\undefined 
\usepackage{MnSymbol} 

\usepackage[shortlabels]{enumitem}
\setlist[itemize]{topsep=.5pt,itemsep=0pt,parsep=2pt}
\setlist[enumerate]{topsep=.5pt,itemsep=0pt,parsep=2pt}

\usepackage[normalem]{ulem}

\newtheorem{thm}{Theorem}
\newtheorem{lem}[thm]{Lemma}

\newtheorem{ass}{Assumption}
\newtheorem{claim}{Claim}
\newtheorem{prop}{Proposition}

\theoremstyle{definition}
\newtheorem{remark}{Remark}
\newtheorem{example}{Example}

\usepackage[toc,page,header]{appendix}
\usepackage{minitoc}
\usepackage{silence}
\usepackage{bibunits}

\title{Crush Optimism with Pessimism: \\Structured Bandits Beyond Asymptotic Optimality}
\author{%
  Kwang-Sung Jun\\
  The University of Arizona\\
  \texttt{kjun@cs.arizona.edu}\\
  \And
  Chicheng Zhang\\
  The University of Arizona\\
  \texttt{chichengz@cs.arizona.edu}\\
}

\begin{document}
\doparttoc 
\faketableofcontents 

\begin{bibunit}[plainnat]

\setlength{\abovedisplayskip}{3pt}
\setlength{\belowdisplayskip}{3pt}
\setlength{\abovedisplayshortskip}{3pt}
\setlength{\belowdisplayshortskip}{3pt}

\let\SS\undefined
\def\SS{{\mathbb{S}}}
\newcommand{\sdef}[2]{{\blue{#1\plabel{#2}}}}

\maketitle

\begin{abstract}
  In this paper,\forarxiv{v2: Added the lower bound result. This version is identical to the NeurIPS'20 camera-ready version.} we study stochastic structured bandits for minimizing regret. 
  The fact that the popular optimistic algorithms do not achieve the asymptotic instance-dependent regret optimality (asymptotic optimality for short) has recently allured researchers.
  On the other hand, it is known that one can achieve a bounded regret (i.e., does not grow indefinitely with $n$) in certain instances.
  Unfortunately, existing asymptotically optimal algorithms rely on forced sampling that introduces an $\omega(1)$ term w.r.t. the time horizon $n$ in their regret, failing to adapt to the ``easiness'' of the instance.
  In this paper, we focus on the finite hypothesis class and ask if one can achieve the asymptotic optimality while enjoying bounded regret whenever possible.
  We provide a positive answer by introducing a new algorithm called CRush Optimism with Pessimism (CROP) that eliminates optimistic hypotheses by pulling the informative arms indicated by a pessimistic hypothesis.
  Our finite-time analysis shows that CROP $(i)$ achieves a constant-factor asymptotic optimality and, thanks to the forced-exploration-free design, $(ii)$ adapts to bounded regret, and $(iii)$ its regret bound scales not with the number of arms $K$ but with an effective number of arms $K_\psi$ that we introduce.
  We also discuss a problem class where CROP can be exponentially better than existing algorithms in \textit{nonasymptotic} regimes.
  Finally, we observe that even a clairvoyant oracle who plays according to the asymptotically optimal arm pull scheme may suffer a linear worst-case regret, indicating that it may not be the end of optimism.
\end{abstract}


\vspace{-1em}
\section{Introduction}
\label{sec:intro}

We consider the stochastic structured multi-armed bandit problem with a fixed arm set.
In this problem, we are given a known structure that encodes how mean rewards of the arms are inter-dependent.
Specifically, the learner is given a space of arms $\cA$ and a space of hypotheses $\cF$ where each $f \in {\cF}$ maps each arm $a\in\cA$ to its mean reward $f(a)$.
Define $[n] := \{1,\ldots,n\}$.
At each time step $t\in[n]$, the learner chooses an arm $a_t \in \cA$ and observes a (stochastic) noisy version of its mean reward $f^*(a)$ where $f^*\in\cF$ is the ground truth hypothesis determined before the game starts and not known to her.
After $n$ time steps, the learner's performance is evaluated by her cumulative expected (pseudo-)regret:
\begin{align}\label{eq:regret}
  \EE \Reg_n = \EE\lt[ n\cd\max_{a\in\cA} f^*(a) -  \sum_{t=1}^n f^*(a_t) \rt]~.
\end{align}
Minimizing this regret poses a well-known challenge in balancing between exploration and exploitation; we refer to Lattimore and Szepesv\'ari~\cite{lattimore20bandit} for the backgrounds on bandits.
We define our problem precisely in Section~\ref{sec:prelim}.

Linear bandits, a special case of structured bandits, have gained popularity over the last decade with exciting applications (e.g., news recommendation)~\cite{auer02using,dani08stochastic,ay11improved,li10acontextual,chu11contextual}.
While these algorithms use the celebrated optimistic approaches to obtain near-optimal \textit{worst-case} regret bounds (i.e., $\wtil{O}(\sqrt{dn})$ where $\wtil{O}$ hides logarithmic factors and $d$ is the dimensionality of the model), Lattimore and Szepesv\'ari \cite{lattimore17theend} have pointed out that their \textit{instance-dependent} regret is often far from achieving the asymptotic instance-dependent optimality (hereafter, asymptotic optimality).
This observation has spurred a flurry of research activities in asymptotically optimal algorithms for structured bandits and beyond, including OSSB~\cite{combes17minimal}, OAM~\cite{hao20adaptive} for linear bandits, and DEL~\cite{ok18exploration} for reinforcement learning, although  structured bandits and their optimality have been studied earlier in more general settings~\cite{agrawal89asymptotically,graves1997asymptotically}.

The asymptotically optimal regret in structured bandits is of order $c(f) \cd \ln(n)$ for instance $f\in \cF$ where $c(f)$ is characterized by the optimization problem in~\eqref{eqn:gamma}.
Its solution $\gamma\in\lbrack0,\infty\rparen^K$ represents the optimal allocation of the arm pulls over $\cA$, and some arms may receive zero arm pulls; we call those with nonzero arm pulls the \textit{informative arms}.
On the other hand, it is well-known that structured bandits can admit bounded regret \cite{bubeck13bounded,lattimore14bounded,atan15global,shen18generalized,guptaunified,tirinzoni2020novel}; i.e., $\limsup_{n\rarrow\infty} \EE \Reg_n < \infty$.
This is because the hypothesis space, which encodes the side information or constraints, can contain a hypothesis $f$ whose best arm alone is informative enough so that exploration is not needed, asymptotically.

However, existing asymptotically optimal strategies such as OSSB~\cite{combes17minimal} cannot achieve bounded regret by design.
The closest one we know is OAM~\cite{hao20adaptive} that can have a sub-logarithmic regret bound.
The main culprit is their forced sampling, a widely-used mechanism for asymptotic optimality in structured bandits~\cite{combes17minimal,hao20adaptive}.
Forced sampling, though details vary, ensures that we pull each arm proportional to an increasing but \textit{unbounded} function of the time horizon $n$, which necessarily forces a \textit{non-finite} regret.
Furthermore, they tend to introduce the dependence on the number of arms $K$ in the regret unless a structure-specific sampling is performed, e.g., pulling a barycentric spanner in the linear structure~\cite{hao20adaptive}.\footnote{
  Some algorithms like OSSB~\cite{combes17minimal} parameterize the exploration rate as $\eps$, introducing $\eps K g_n$ for some $g_n = \omega(1)$ in the regret bound. One may attempt to set $\eps = 1/K$ to remove the dependence, but there is another term $K/\eps$ in the bound (see \cite[Appendix 2.3]{combes17minimal}). Above all, we believe the dependence on $K$ has to appear somewhere in the regret if forced sampling is used.
}
While the dependence on $K$ disappears as $n\rarrow\infty$, researchers have reported that the lower-order terms do matter in practice~\cite{hao20adaptive}.
Such a dependence also goes against the well-known merit of exploiting the structure that their regret guarantees can have a mild dependence on the number of arms or may not scale with the number of arms at all (e.g., the worst-case regret of linear bandits mentioned above).
We discuss more related work in the appendix (found in our supplementary material) due to space constraints, though important papers are discussed and cited throughout.\footnote{
  Concurrent studies by \citet{degenne20structure} and \citet{saber20forced} avoid forced sampling but still have an explicit linear dependence on $K$ in the regret.
}

Towards adapting to the easiness of the instance while achieving the asymptotic optimality, we turn to the simple case of the finite hypothesis space (i.e., $ |\cF| <\infty$) and ask: \textit{can we design an algorithm with a constant-factor asymptotic optimality while adapting to finite regret?}
Our main contribution is to answer the question above in the affirmative by designing a new algorithm and analyzing its finite-time regret.
Departing from the forced sampling, we take a fundamentally different approach, which we call CROP (CRush Optimism with Pessimism).
In a nutshell, at each time step $t$, CROP maintains a confidence set $\cF_t \subseteq \cF$ designed to capture the ground truth hypothesis $f^*$ and identifies two hypothesis sets: the optimistic set $\tilcF_t$ and the pessimistic set $\barcF_t$ (defined in Algorithm~\ref{alg:crop}).
The key idea is to first pick carefully a $\barf_t\in\barcF_t$ that we call ``pessimism'', and then pull the \textit{informative arms} indicated by $\barf_t$.
This, as we show, eliminates either the optimistic set $\tilcF_t$ or the pessimism $\barf_t$ from the confidence set.
Our analysis shows that repeating this process achieves the asymptotic optimality within a constant factor.
Furthermore, our regret bound reduces to a finite quantity whenever the instance allows it and does not depend on the number of arms $K$ in general; rather it depends on an effective number of arms $K_\psi$ defined in~\eqref{eq:Kpsi}.
We elaborate more on CROP and the role of pessimism in Section~\ref{sec:crop}.
We present the main theoretical result in Section~\ref{sec:analysis} and show a particular problem class where CROP's regret bound can be exponentially better than that of forced-sampling-based ones.
Our regret bound of CROP includes an interesting $\ln\ln(n)$ term. 
In Section~\ref{sec:lb}, we show a lower bound result indicating that such a $\ln\ln(n)$ term is unavoidable in general.

Finally, we conclude with discussions in Section~\ref{sec:discussion} where we report a surprising finding that UCB can be in fact better than a clairvoyant oracle algorithm (that, of course, achieves the asymptotic optimality) in nonasymptotic regimes.
We also show that such an oracle can suffer a linear worst-case regret under some families of problems including linear bandits, which we find to be disturbing, but this leaves numerous open problems.

\section{Problem definition and preliminaries}
\label{sec:prelim}

In the structured multi-armed bandit problem, the learner is given a discrete arm space $\sdef{\cA}{s:cA} = [K]$, and a finite hypothesis class $\sdef{\cF}{s:cF} \subset ( \cA \to \RR )$ where we color definitions in blue, hereafter.
There exists an unknown $\sdef{f^*}{s:fstar}\in\cF$ that is the ground truth mean reward function.
Denote by $n$ the time horizon of the problem.
For every $f\in \cF$, denote by $\sdef{a^*(f)}{s:astar} = \argmax_{a \in \cA} f(a)$ and $\sdef{\mu^*(f)}{s:mustar} = \max_{a \in \cA} f(a)$ the arm and the mean reward {\em supported by} $f$, respectively.
We remark that the focus of our paper is not computational complexity but the achievable regret bounds.
For ease of exposition, we make the unique best arm assumption as follows:\footnote{
  Our algorithms and theorems can be easily extended to the setting where optimal actions w.r.t. $f$ can be non-unique. This requires us to redefine the equivalence relationship, which we omit for brevity.
}
\begin{ass}[Unique best arm]\label{ass:unique}
  For every $f\in\cF$, there exists a unique best arm $a^*(f)$, i.e., $a^*(f)$ is singleton.
\end{ass}
\vspace{-.3em}
For an arm $a$ and a hypothesis $f$, denote by $\sdef{\Delta_a(f)}{s:star} = \mu^*(f) - f(a)$ the gap between the arm $a$ and the optimal arm, if the true reward function were $f$.
Given a set of hypotheses $\cG$, we denote by $\sdef{a^*(\cG)}{s:astar-cG} = \cbr[0]{  a^*(f): f \in \cG }$ and $\sdef{\mu^*(\cG)}{s:mustar-cG} = \cbr[0]{  \mu^*(f): f \in \cG }$ the set of arms and mean rewards supported by $\cG$ respectively.

The learning protocol is as follows: for each round $t\in[n]$, the learner pulls an arm $\sdef{a_t}{s:at} \in \cA$ and then receives a reward $\sdef{r_t}{s:rt} = f^*(a_t) + \xi_t$ where $\xi_t$ is an independent $\sig^2$-sub-Gaussian random variable.
The performance of the learner is measured by its expected cumulative regret over $n$ rounds defined in~\eqref{eq:regret}.
Given an arm $a$ and time step $t$, denote by $\sdef{T_a(t)}{s:Ta} = \sum_{s=1}^t \one\cbr{a_s = a}$ the  arm pull count of $a$ up to round $t$.
With this notation, $\EE \Reg_n = \sum_{a \in \cA} \EE[T_a(n)] \Delta_a(f^*)$. 

\paragraph{Asymptotically optimal regret.}
Our aim is to achieve an asymptotic instance-dependent regret guarantee.
Hereafter we abbreviate `asymptotic optimality' to \textbf{AO}.
Specifically, we would like to develop {\em uniformly good} algorithms, in that for any problem instance, the algorithm satisfies $\EE \Reg_n = o(n^p)$ for any $p > 0$ where the little-o here is w.r.t. $n$ only.
The regret lower bound of structured bandits is based on the \textit{competing} class of functions $\sdef{\cC(f)}{s:cC}= \cbr{g: g(a^*(f)) = f(a^*(f)) \wedge a^*(g) \neq a^*(f)}$.
The class $\cC(f)$ consists of hypotheses $g\in\cF$ such that pulling arm $a^*(f)$ provides no statistical evidence to distinguish $g$ from $f$.
Thus, even if the learner is confident that $f$ is the ground truth, she has to pull arms other than $a^*(f)$ to guard against the case where the true hypothesis is actually $g$ (in which case she suffers a linear regret); see the example in Figure~\ref{fig:diagram}(a) where $\cC(f_4) = \{f_1\}$.
The lower bound precisely captures such a requirement as constraints in the following optimization problem:
\begin{align}
\label{eqn:gamma}
\blue{c(f)} := \min_{\gam \in \lbrack0,\infty\rparen^K: ~ \gam_{a^*(f)} = 0} ~~~ \sum_a \gam_a \Delta_a(f)
~~~~\text{ s.t. } ~~~ \forall g \in \cC(f),  ~~
\sum_a \gam_a \cd \KL(f(a),g(a)) \ge 1 ~.
\end{align}
where $\KL(f(a),g(a))$ is the KL-divergence between the two reward distributions when the arm $a$ is pulled under $f$ and $g$ respectively.
For the discussion of optimality, we focus on Gaussian rewards with variance $\sig^2$, which means $\KL(f(a),g(a)) = \fr{(f(a)-g(a))^2}{2\sig^2}$, though our proposed algorithm has a regret guarantee for more generic sub-Gaussian rewards.
We denote by $\sdef{\gam(f)}{s:gam}$ the solution of~\eqref{eqn:gamma}.
Then, $c(f) = \sum_{a \in \cA} \gam_a(f)\cd\Delta_a(f)$.
The intuition is that if one could play arms in proportion to $\sdef{\gam^*}{s:gamstar} = \cbr{\gam_a(f^*)}_{a \in \cA}$, then, by the constraints of the optimization problem, she would have enough statistical power to distinguish $f^*$ from all members of $\cC(f^*)$; furthermore, $\gamma^*$ is the most cost-efficient arm allocation due to the objective function.
The value of $\gam^*$ can be viewed as the allocation that balances optimally between maximizing the \textit{information gap} (i.e., the KL divergence in~\eqref{eqn:gamma}) and minimizing the \textit{reward gap} (i.e., $\Delta_a(f)$).

It is known from the celebrated works of Agrawal et al.~\cite{agrawal89asymptotically} and Graves and Lai~\cite{graves1997asymptotically} that any uniformly good algorithm must have regret at least $(1 - o(1)) c(f) \ln (n)$ for large enough $n$, under environment with ground truth reward function $f^* = f$.
In other words, if an algorithm has a regret of $(1 - \Omega(1)) c(f) \ln n$ under the ground truth $f$, then for large enough $n$, its expected arm pull scheme $\gamma = \del[0]{\frac{\EE\sbr{T_a(n)}}{\ln n}}_{a \in \cA}$ must violate the constraint in~\eqref{eqn:gamma} for some $g \in \cC(f)$, implying that the algorithm must not be a uniformly good algorithm (i.e., suffer a polynomial regret under $g$).
They also show the lower bound is tight by developing algorithms with asymptotic regret bound of $(1 + o(1)) c(f) \ln n$.

\textbf{The oracle.} The lower bound suggests that one should strive to ensure $\EE[T_a(n)] \approx \gam^*_a \ln(n)$.
Indeed a clairvoyant oracle (the oracle, hereafter) who knows $f^*$ would, at round $t$, pull the arm $a$ such that $T_a(t-1) \le \gam^*_a\ln(t)$ if there exists such an arm (i.e., exploration), and otherwise pull the best arm (i.e., exploitation).
The oracle will initially pull the informative arms only, but as $t$ increases, exploitation will crowd out exploration.
We believe mimicking the oracle is what most algorithms with AO are after.
Particularly, the most common strategy is to replace $\gam^*$ with the Empirical Risk Minimizer (ERM) $\gam(\hf_t)$ where $\hf_t \in\cF$ is the one that best fits the observed rewards.
Unlike supervised learning, however, the observed rewards are controlled by the algorithm itself, making the ERM brittle; i.e., the ERM may not converge to $f^*$.
Thus, most studies employ a form of forced sampling to ensure that $\hf_t$ converges to $f^*$ so that $\gam(\hf_t)$ converges to $\gam^*$.
As discussed before, this is precisely where the issues begin, and we will see that CROP avoids forced sampling and $\gam(\hf_t)$ altogether.

\textbf{Example: cheating code.}
We describe an example inspired by~\citet{pmlr-v19-amin11a} when algorithms with AO provide an improvement over the popular optimistic algorithms. 
Let $K_0 \in\NN_+$ and $e_i$ be the $i$-th indicator vector.
The idea is to first consider a hypothesis like $f = (1, 1-\eps, 1-\eps, \ldots, 1-\eps)$ and then add those hypotheses that copy $f$, pick one of its non-best arms, and replace its mean reward with $1+\eps$.
This results in total $K_0-1$ competing hypotheses.
Specifically, let $\blue{e_i}$ be the $i$-th indicator vector and define $h(i,j) \in \RR^{K_0}$ as follows: $\forall i\in[K_0], h(i,0) = (1-\eps) \mathbf{1} + \eps e_i$ and $\forall j \in [K_0]\sm \{i\},~~ [h(i,j)]_k = \begin{cases}
1+\eps & \text{if $k=j$}
\\  h(i,0) & \text{otherwise} \end{cases}$.
Let $\cF_0 = \cbr{h(i,j): i \in [K_0], j \in \cbr{0,1,\ldots,K_0} \sm \cbr{i}}$, $k= \lceil \log_2(K_0) \rceil$, and $\Lam\in[0,1/2]$.
Finally, we define the ``cheating code'' class with $K=K_0 + k$ arms:
\begin{align*}
&\blue{\cF^{\text{code}}} = \{(g_{1:K_0}, \Lam\cd b_{1:k})\in\RR^{K_0+k}: g \in \cF_0, b\in\{0,1\}^k \text{: binary representation of } a^*(g)-1\}~,
\end{align*}
which appends $k$ ``cheating arms'' that tells us the index of the best arm.
Let us fix $f^* \in \cF^{\code}$ such that $\mu^*(f^*) = 1$.
Assume $\fr{1}{2\eps} > \fr{2}{\Lam^2}$ so that the informative arms of $f^*$ are the cheating arms (see the appendix for reasoning) where we color in green for emphasis, throughout.
Let $\sig^2 = 1$.
For the instance $f^*$, an algorithm with a constant-factor AO would have regret $O(\fr{\log_2 K}{\Lam^2}\ln(n))$ (elaborated more in the appendix).
In contrast, optimistic algorithms such as UCB~\cite{auer02finite} (i.e., run naively without using the structure) or UCB-S~\cite{lattimore14bounded}, would pull the arm $\tila_t$ where
\begin{align}\label{eq:optimism}
(\tila_t, \tilf_t) = \argmax_{a\in\cA, f\in\cF_t} f(a)
\end{align}
and $\cF_t$ is a confidence set designed to trap $f^*$ with high probability.
One can show that $\tila_t$ is always one of the first $K_0$ arms and that their regret is $O(\fr{K}{\eps}\ln(n))$, which can be much larger.
In fact, the gap between the two bounds can be arbitrarily large as $\eps$ approaches to $0$.

\begin{figure}
  \centering
  \begin{tabular}{cc}
    \includegraphics[width=.35\linewidth,valign=t]{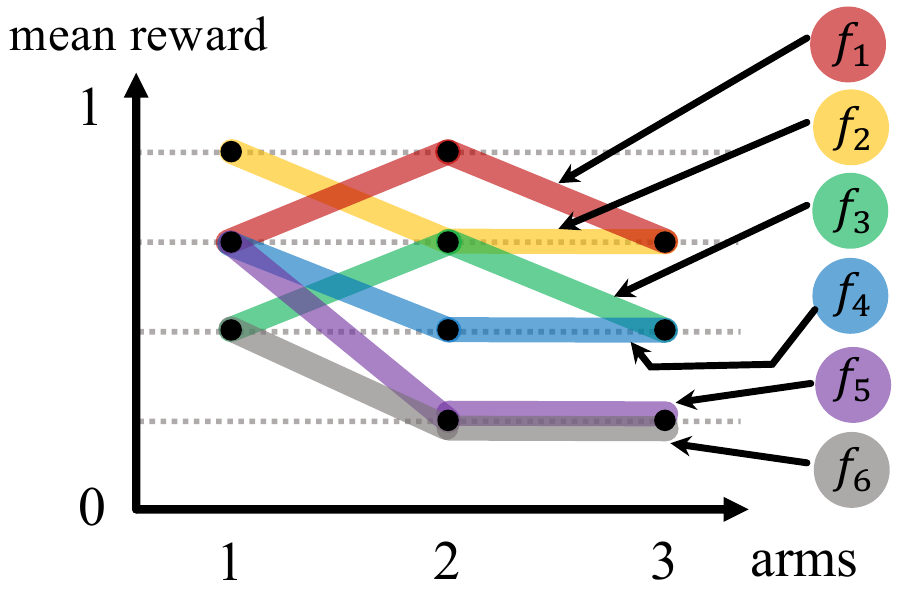}&  \includegraphics[width=.35\linewidth,valign=t]{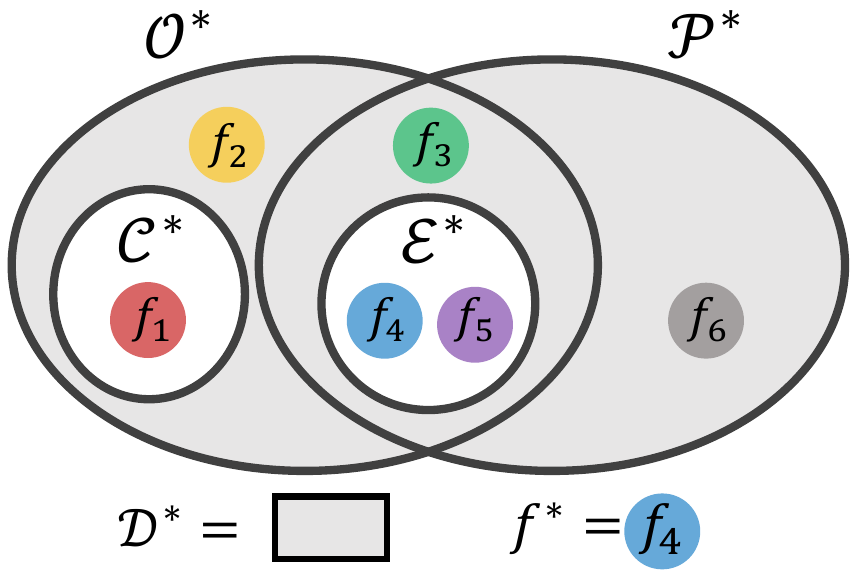}
    \\ (a) & (b)
  \end{tabular}
  \vspace{-.4em}
  \caption{(a) An example instance. (b) A diagram of various hypothesis classes w.r.t. the ground truth hypothesis $f^*$. Best viewed in colors.}
  \label{fig:diagram}
  \vspace{-0.0em}
\end{figure}

\paragraph{The anatomy of the function classes.}
\!\!There are function classes besides $\cC(f)$ that will become useful in our study.
We first define an equivalence relationship between hypotheses: we call $\sdef{f \sim g}{s:fsimg}$ if $a^*(f) = a^*(g)$ and $\mu^*(f) = \mu^*(g)$; one can verify that it satisfies reflexiveness, symmetry, and transitivity, and induces a partition over $\cF$.
Given $f\in\cF$, we denote by $\sdef{\cE(f)}{s:cE}$ the equivalent class $f$ belongs to and by $\sdef{\cD(f)}{s:cD} = \cbr{g: g(a^*(f)) \neq f(a^*(f))}$ its {\em docile} class that can be easily distinguished from $f$ as we describe later.
One can show that for every $f\in\cF$, the class $\cF$ is a disjoint union of $\cE(f)$, $\cD(f)$, and $\cC(f)$.
We also define $\sdef{\cO(f)}{s:cO} = \cbr{g: \mu^*(g) \geq \mu^*(f)}$ (and $\sdef{\cP(f)}{s:cP} = \cbr{g: \mu^*(g) \leq \mu^*(f)}$) as the set of hypotheses that support mean rewards that are not lower (and not higher) than $\mu^*(f)$ (respectively).
We use shorthands $\cE^*:=\cE(f^*)$ and $\cD^*, \cC^*, \cO^*$, and $\cP^*$  defined similarly.
We draw a Venn diagram of theses classes in Figure~\ref{fig:diagram}(b) along with the example hypotheses in Figure~\ref{fig:diagram}(a); we recommend that the readers verify the example themselves to get familiar with these classes.

\paragraph{Bounded regret.}
When the ground truth $f^*$ enjoys $\cC^* = \emptyset$, then $c(f^*)=0$ and the algorithm can achieve bounded regret, which is well-known as mentioned in our introduction. 
This is because, when $f^*$ has no competing hypothesis, pulling the best arm $a^*(f^*)$ alone provides a nonzero statistical evidence that distinguishes $f^*$ from $\cF \sm \cE^* = \cD^*$.
That is, there is no need to explore as exploitation alone provides sufficient exploration.

\vspace{-.5em}
\section{Crush Optimism with Pessimism (\crop)}
\label{sec:crop}
\vspace{-.5em}

\begin{algorithm}[t]
  \caption{CRush Optimism with Pessimism (\crop)}
  \label{alg:crop}
  {\small
  \begin{algorithmic}[1]
    \REQUIRE{The hypothesis class $\cF$, parameters $z, \rz$, $\alpha, \ralpha > 1$ }
    \FOR{$t=1,2,\ldots,n$}
    \STATE Let
    $\sdef{\cF_t}{s:cFt} = \{f\in\cF: L_{t-1}(f) - \min_{g\in\cF}L_{t-1}(g) \le \sdef{\beta_t}{s:beta} := 4\sig^2 \ln\lt(zt^\alpha\rt)\}$ 
    \IF{$a^*(\cF_t)$ is singleton}
    \STATE \textbf{(\exploit)} Pull the arm $a_t \in a^*(\cF_t)$, observe the reward $r_t$.
    \label{line:exploit}
    \STATE Continue to the next iteration.
    \ENDIF
    \STATE Let $\sdef{\cB_t}{s:cB} = \{(a^*(f),\mu^*(f)): f \in \cF_t\}$ be the \textbf{b}est arm candidate set. 
    \STATE Find the optimistic arm, mean, and set: \plabel{s:tila} \plabel{s:tilmu} \plabel{s:tilcF}
    \phantom{\footnotesize$\sdef{\tila_t}{s:tila}, \sdef{\tilmu_t}{s:tilmu}    \sdef{\tilcF_t}{s:tilcF}$  }
    \vspace{-.5em}
    \begin{align*}
    (\blue{\tila_t}, \blue{\tilmu_t}) = \argmax_{(a,\mu) \in \cB_t}~ \mu, ~~~~ {\blue{\tilcF_t}} = \cF_t(\tila_t, \tilmu_t)~.
    \end{align*}
    \vspace{-1em}
    \STATE Find the pessimistic arm, mean, set, and hypothesis:
    \plabel{s:barf} \plabel{s:barcF} \plabel{s:barmu} \plabel{s:bara}
    \vspace{-.5em}
    \begin{align*}
    (\blue{\bara_t}, \blue{\barmu_t}) = \argmin_{(a,\mu) \in \cB_t:~ a \neq \tila_t}~ \mu
    ,~~~~ \blue{\barcF_t} = \cF_t(\bara_t, \barmu_t)
    ,~~~~ \blue{\barf_t} = \argmin_{f\in\barcF_t} L_{t-1}(f)~.
    \end{align*}  \vspace{-1em}
    \label{line:barf}
    \STATE Define $\sdef{\rcF_t}{s:rcF} = \{f \in \barcF_t: L_{t-1}(f) -  L_{t-1}(\barf_t) \le \sdef{\rbeta_t}{s:rbeta} := 4\sig^2\ln(\rz(\log_2(t))^{\ralpha})\}$.  \hfill{} (let $\rbeta_1=\infty$)
    \label{line:cf-elim}
    \IF{$\exists f, g \in\mathring{\cF}_t \suchthat \gam(f) \not\propto \gam(g)$}
    \STATE \textbf{(\conflict)} $\pi_t = \phi(\barf_t)~.$ \hfill{} (see~\eqref{eq:phi})
    \label{line:cf}
    \ELSIF{ $\gam(\barf_t)$ satisfies that $\forall f\in\tilcF_t, \sum_a \gam_a(\barf_t)\fr{(\barf_{t}(a) - f(a))^2}{2\sig^2} \ge1$,}
    \label{line:fs-cond}
    \STATE \textbf{(\feasible)} $\pi_t = \gam(\barf_t)$.
    \label{line:fs}
    \ELSE
    \STATE \textbf{(\fallback)} $\pi_t = \psi(\barf_t) $. \hfill{} (see~\eqref{eq:psi})
    \label{line:fb}
    \ENDIF
    \STATE Pull arm $a_t = \argmin_{a} \fr{T_a(t-1)}{\pi_{t,a}}$~~~~ (take $\fr x 0$ with $x\ge0$ as $\infty$; break ties arbitrarily), and then observe the reward $r_t$.
    \label{line:armpull}
    \ENDFOR
  \end{algorithmic}
}
\end{algorithm}
\textfloatsep=.5em
We now introduce our algorithm CROP.
First, some definitions: for any $\cG$, define $\sdef{\cG(a,\mu)}{s:cG-pair} = \{f\in\cG: a^*(f) = a, \mu^*(f) = \mu\}$.
Given a set of observations $\cbr{(a_s, r_s): s \in [t]}$ up to time step $t$, and $f\in\cF$, denote by $\sdef{L_t(f)}{s:Lt} = \sum_{s=1}^t (f(a_s) - r_s)^2$ the cumulative squared loss of $f$ up to time step $t$.
We use this loss to construct a confidence set that captures the ground truth $f^*$, which is inspired by Agarwal et al. \cite{agarwal11stochastic}, but we extend theirs to allow sub-Gaussian rewards.
The loss $L_t(f)$ gives a measure of goodness of fit of hypothesis class $f$, in that $f^*$ is the Bayes optimal regressor that minimizes $\EE L_t(f)$.

We describe CROP in Algorithm~\ref{alg:crop}, where the parameters $\{\alpha,\ralpha\}$ are numerical constants and $\{z,\rz\}$ should be set to $|\cF|$ (precise defined in \Cref{thm:main}).
CROP has four main branches: \exploit, \conflict, \feasible, and \fallback.
Note that \feasible is the main insight of the algorithm that we focus first while \conflict deals with some difficult cases, which we describe the last.

\paragraph{\exploit.}
\!\!\!\!  At every round $t$, CROP maintains a confidence set $\cF_t$, the set of hypotheses $f$ in $\cF$ that fits well with the data observed so far w.r.t. $L_{t-1}(f)$.
This is designed so that the probability of failing to trap the ground truth hypothesis $f^*$ is $\order{\frac{1}{t^\alpha}}$.
We first check if $a^*(\cF_t)$ is a singleton.
If true, we pull the arm $a^*(\cF_t)$ that is unanimously supported by all $f$ in $\cF_t$.
Note that the equivalence relationship $\sim$ induces a partition of $\cF_t$.
If we do not enter the exploit case, we select the equivalence class $\tilcF_t$ that maximizes its shared supported mean reward; we call this the {\em optimistic set}.
This is related to the celebrated``optimism in the face of uncertainty'' (OFU) principle that pulls arm $a_t$ by \eqref{eq:optimism}.
In line~\ref{line:barf}, we deviate from the OFU and define the  {\em pessimistic set}  $\barcF_t$, which is the equivalence class in $\cF_t$ that \textit{minimizes} its shared supported mean reward $\mu^*(\barcF_t)$ with a constraint that they support an action other than $\tila_t$.
We then define $\barf_t$, which we call \textit{the pessimism}, as the Empirical Risk Minimizer (ERM) over $\barcF_t$. 
Next, we compute $\rcF_t$, a refined confidence set inside the pessimistic set $\barcF_t$, and then test a condition to enter \conflict; we will discuss it later as mentioned above.
For now, suppose that we did not enter \conflict and are ready to test the condition for \feasible (line~\ref{line:fs-cond}).

\paragraph{\feasible.}
The condition in line~\ref{line:fs-cond} first computes $\gam(\barf_t)$ and then tests whether all the hypotheses in $\tilcF_t$ satisfy the information constraint that takes the same form as those in the optimization problem for $c(\barf_t)$.
If this is true, then we set $\pi_t=\gam(\barf_t)$ and then move onto line~\ref{line:armpull} to choose which arm to pull.
The intention here is to pull the arm that is most far away from the pull scheme of $\gam(\barf_t)$, which is often referred to as ``tracking''~\cite{garivier16optimal}.
Note that the arm $\bara_t$ is never pulled because $\gam_{\bara_t}(\barf_t) = 0$.

\begin{wrapfigure}{R}{0.4\linewidth}
  \vspace{-1.2em}
  \begin{minipage}{0.4\textwidth}
    \begin{tabular}{C{5ex}ccccc}\hline
      Arms & A1 & A2 & A3 & A4 & A5 \\ \hline
      $f_1$ & \textbf{1}  & .99 &  .98 & 0 & 0 \\
      $f_2$ & .98 & \textbf{.99} & .98 & \underline{.25} & 0 \\
      $f_3$ & .97 & .97 & \textbf{.98} & .25 & \underline{.25}  \\   \hline
      $f_4$ \newline{\!\!\footnotesize(optional)}
              & .98 & \textbf{.99} & .98 &.25 & \underline{.50} \\ \hline
    \end{tabular}
    \vspace{-.4em}
    \caption{The ``staircase'' example. Define $\blue{\cH}=\{f_1,f_2,f_3\}$ and $\blue{\cH^+}=\{f_1,f_2,f_3,f_4\}$.
      We boldface the best arm and underline the informative arms of each hypothesis. }
    \label{fig:staircase}
  \end{minipage}
  \vspace{-1.3em}
\end{wrapfigure}

\paragraph{Why the pessimism?}~
To motivate the design choice of tracking the pessimism, consider the example hypothesis space $\cH$ in Figure~\ref{fig:staircase}.
Suppose that at time $t$ we have $\cF_t = \cH = \{f_1, f_2, f_3\}$.
Which arms should we pull?
The OFU tells us to pull the optimistic arm $\tila_t$ as done in Lattimore and Munos \cite{lattimore14bounded}, but it does not achieve the instance optimality.
Another idea mentioned in Section~\ref{sec:prelim} is find the ERM $\sdef{\hf_t}{s:hft} = \argmin_{f\in\cF} L_{t-1}(t)$ and then pull the arms by tracking $\gam(\hf_t)$; i.e., $a_t = \argmin_{a\neq a^*(\hf_t)} T_a(t-1) / \gam_a(\hf_t)$.
This is essentially the main idea of OSSB~\cite{combes17minimal}.\footnote{
  OSSB in fact does not find the ERM but rather uses the empirical means of the arms to solve the optimization problem~\eqref{eqn:gamma}, which can work for some problem families.
  Still, we believe extending OSSB to use the ERM with suitable loss function should achieve (near) asymptotic optimality for the finite $\cF$.
}
OAM~\cite{hao20adaptive} also relies on the ERM $\hf_t$, though they partly use the optimism.
However, ERMs are brittle in bandits. 
For example, when $f^* = f_3$, in earlier rounds, the ERM $\hf_t$ can be $f_2$ with nontrivial probability.
Pulling the informative arm of $f_2$, which is A4, eliminates $f_1$ but will not eliminate $f_3$, and we get stuck at pulling A4 indefinitely.
To avoid such a trap, researchers have introduced forced sampling.

What are the robust alternatives to the ERM?
For now, suppose that $f^*$ is always in the confidence set $\cF_t$.
Among $\{\gam(f_1), \gam(f_2), \gam(f_3)\}$, which one should we track?
We claim that we should follow the pessimism, which is $f_3$ in this case.
Specifically, if $f^*=f_3$, we are lucky and following the pessimism will soon remove both $f_1$ and $f_2$ from $\cF_t$.
We then keep entering \exploit and pull the best arm A3 for a while.
Note that $f_1$ or $f_2$ will come back to $\cF_t$ again as pulling A3 provides the same loss to every $f\in\cH$ but the threshold $\beta_t$ of the confidence set $\cF_t$ increases over time.
In this case, the principle of pessimism will do the right thing, again.

What if $f^*$ was actually $f_2$?
Following the pessimism $f_3$ is not optimal, but it does eliminate $f_3$ from $\cF_t$ because $f_2$ appears in the constraint of the optimization~\eqref{eqn:gamma}; after the elimination, we have $\cF_t=\{f_1,f_2\}$ and $\barf_t=f_2$, so the pessimism is back in charge.
In sum, the key observation is that the optimal pull scheme $\gam(f)$ is designed to differentiate $f$ from its \textit{competing hypotheses} that support arms with higher mean rewards than that of $f$.
Assuming the confidence set works properly, tracking the pessimism either does the right thing or, if $\barf_t$ is not the ground truth, removes $\barf_t$ from the confidence set (also the right thing to do).
However, to make it work beyond this example, we need other mechanisms: \fallback and \conflict.

\paragraph{\fallback.}
When the condition of \feasible is not satisfied, we know that the arm pull scheme $\gam(\barf_t)$ will not be sufficient to remove every $f\in\tilcF_t$ -- or, it could even be impossible.
Thus, we should not track $\gam(\barf_t)$.
Instead, we design a different arm pull scheme $\sdef{\psi(f)}{s:psi}$ defined below so that tracking $\psi(\barf_t)$ can remove all members of $\tilcF_t$ in a cost-efficient manner.
With the notation $\Delta_{\min}(f) = \min_{a\neq a^*(f) } \Delta_{a^*(f)}(f)$,
\begin{align}
\begin{aligned}\label{eq:psi}
\psi(f) := \argmin_{\gam \in \lbrack 0,\infty\rparen^K}~~~
    &\Delta_{\min}(f) \cd \gam_{a^*(f)} +  \sum_{a\neq a^*(f)} \Delta_a(f) \cd\gam_a
\\ ~~\text{ s.t. }~~  &
\forall g\in\cO(f) \sm \cE(f):~~ \sum_a \gam_a \fr{(f(a) - g(a))^2}{2\sig^2}\ge 1
\\ &\gam \succeq \phi(f) \vee \gam(f)
\end{aligned}
\end{align}
where $\phi$ is defined in~\eqref{eq:phi} and explained below and $\blue{x \succeq y}$ means $x_a \ge y_a,\forall a$.
The constraints above now ensure that $\psi(\barf_t)$ provides a sufficient arm pull scheme to eliminate $\tilcF_t$ even if the condition of \feasible is not satisfied.
Another difference from $\gam(f)$ is that $\gam_{a^*(f)}$ can be nonzero, but we use $\Delta_{\min}(f)$ instead of $\Delta_{a^*(f)}(f) = 0$ to avoid $\gam_{a^*(f)}=\infty$.
That said, there are other design choices for $\psi(f)$, especially given that $\psi$ appears only with the finite terms in the regret bound.
We discuss more on the motivation and alternative designs of~\eqref{eq:psi} in the appendix.

\paragraph{\conflict.}
This is an interesting case where the learner faces the challenge not in finding which arm is the best arm, but rather which \textit{informative} arms and their pull scheme one should track.
Specifically, consider the other example of $\cH^+$ in Figure~\ref{fig:staircase}.
Suppose at time $t$ we have $\cF_t=\cbr{f_1, f_2, f_4}$ and the ground truth is $f_4$, which means $\cE^* = \cbr{f_2,f_4}$.
If CROP does not have the \conflict mechanism, it will use $\barf_t$, the ERM among $\barcF_t$, which can be either $f_2$ or $f_4$.
However, as explained before, ERMs are brittle; one can see that it can get stuck at tracking $f_2$ with nontrivial probability and pull less informative arms.
Interestingly, this would not incur a linear regret.
Rather, the regret would still be like $\ln(n)$ but with a suboptimal constant of $c(f_2)$ rather than $c(f_4)$; one can adjust our example to make this gap $c(f_2)-c(f_4)$ arbitrarily large, making it arbitrarily far from the AO. 
On the other hand, a closer look at $f_2$ and $f_4$ reveals that A5 is the only arm that can help distinguish $f_2$ from $f_4$.
One might attempt to change CROP so that it pulls A5 in such a case, which results in either removing $f_4$ from the confidence set if $f^*=f_2$ or removing $f_2$ if $f^*=f_4$.
However, if $f^*=f_2$, this would introduce an extra $\ln(n)$ term in the regret bound since A5 is a noninformative arm, which again can lead to a suboptimal regret bound. 

CROP resolves this issue by constructing a refined confidence set $\rcF_t$ with a more aggressive failure rate of $1/\ln(t)$ rather than the usual $1/t$, and use this confidence set to weed out conflicting pull schemes.
If the refined set $\rcF_t$ still contains hypotheses with conflicting pull schemes, then CROP enters \conflict and computes a different allocation scheme:
    \begin{align}
  \label{eq:phi}
  \begin{aligned}
    \blue{\phi(f)} =   \argmin_{\gam \in \lbrack 0,\infty\rparen^K:  \gam_{a^*(f)} = 0} &~~\sum_a \Delta_a(f) \cd\gam_a
  \\
  \text{ s.t. } &\forall g\in\cE(f): \gam(g) \not\propto\gam(f) ,~~
  \sum_a \gam_a \fr{(f(a) - g(a))^2}{2\sig^2}\ge 1
  \end{aligned}
  ~,
  \end{align}
where we use the convention $0\propto0$.
Consider $\cH^+$ in Figure~\ref{fig:staircase} with $\sig^2 = 1$.
Then, $\phi(f_2) = \phi(f_4) = (0,0,0,0,\fr{2}{(.5)^2}=8)$.
Our regret analysis will show that the quantity $\phi(f)$ appears in the regret bound with $\ln(\ln(n))$ term only instead of $\ln(n)$, allowing us to achieve the AO within constant-factor.

\vspace{-.5em}
\section{Analysis}
\label{sec:analysis}
\vspace{-.5em}

Before presenting our analysis, we define the effective number of arms $\sdef{K_\psi}{s:Kpsi}$ as the size of the union of the supports of $\psi(f)$ for all $f\in\cF$:
\begin{align}\label{eq:Kpsi}
K_\psi = \abs{\cbr{a: \exists f \in\cF, \psi_a(f) \neq 0}}~.
\end{align}
Define $\blue{\phi(\cG)} = \del{\max_{f \in \cG} \phi_a(f)}_{a\in\cA}$ and $\blue{\psi(\cG)}$ similarly.
Let $\blue{\Lam_{\min}} = \min_{f \in\cD^*} \fr{| f(a^*) - f^*(a^*) |}{\sig}$ the smallest information gap where $a^* := a^*(f^*)$.
We use the shorthand $\Delta_{\max} := \max_{a} \Delta_a(f^*)$.
We present our main theorem on the regret bound of CROP.
\vspace{-.2em}
\begin{thm}  \label{thm:main}
  Let $(\alpha,\ralpha,z,\rz) = (2,3,|\cF|,|\cF|)$.
  Suppose we run \crop with hypothesis class $\cF$ with the environment $f^*\in\cF$.
  Then, \crop has the following anytime regret guarantee: $\forall n\ge2$,
  \begin{align*}
    \EE\Reg_n
    &\le c_1 \cd \del{ P_1 \ln(n) + P_2 \ln(\ln(n))
    + P_3 \del{\ln(|\cF|)  +  \ln\del{ Q_1 }}
    + K_\psi \Delta_{\max} }~,
  \end{align*}
  where $c_1$ is a numerical constant, and
  \begin{align*}
  P_1 = \sum_a \Delta_a \gam^*_a,
  ~  P_2 = \sum_{a} \Delta_a \phi_a(\cE^*),
  ~ P_3 = \sum_a \Delta_a \psi_a({\cF}),
  ~\text{ and } ~ Q_1= \Lam_{\min}^{-2} + K_\psi(1+\max_i \psi_i({\cF}))~.
  \end{align*}
  Furthermore, when $\gam^* = 0$, we have $P_1=P_2=0$, achieving a bounded regret.
\end{thm}
\vspace{-.8em}
\begin{proof}
  The main proof is deferred to the appendix.
  One technical challenge is to deal with \conflict in CROP via our refined confidence set $\rcF_t$.
  The failure rate of $\rcF_t$ is set $\poly(1/\ln(t))$ rather than the usual $\poly(1/t)$.\footnote{Similar aggressive definitions of confidence sets have also appeared in recent works for other purposes~\cite[e.g.][]{lattimore17theend}.}
  For example, there is an event where $\rcF_t$ fails to capture $f^*$ but $f^*$ is still in $ \cF_t$, which would lead to a $\ln(n)$ regret; we manage to prove that this scenario contributes to an $O(1)$ term \textit{in expectation} by showing that it happens with probability like $1/\ln(n)$ times (roughly speaking) using a technique that we call ``regret peeling''.
  To bound other $O(1)$ terms that are attributed to the docile class $\cD^*$, we borrow techniques from Lattimore and Munos~\cite{lattimore14bounded}.
\end{proof}

\vspace{-.7em}
Our main theorem provide a sharp non-asymptotic instance-dependent regret guarantee.
The leading term $\order{\sum_a \Delta_a \gam^*_a \log(n)}$ implies that we achieve the AO up to a constant factor.
The second term is of order $\ln\ln(n)$, which comes from our analysis on \conflict.
The remaining terms are $\order{1}$, which depends on properties of  $\psi_a({\cF})$ and $\Lam_{\min}$.
Unlike many strategies that perform forced exploration on all arms~\cite{hao20adaptive,combes17minimal} to achieve the asymptotic optimality, our bound has no dependency on the number of arms $K$ at all, even in the finite terms, but rather depends on the effective number of arms $K_\psi$.

Note that $K$-free regret bounds still happens with optimistic algorithms; e.g., in $\cF^{\text{code}}$ (defined in Section~\ref{sec:prelim}), UCB depends on $K_0$ rather than $K$, and one can add arbitrarily many cheating arms to make $K \gg K_0$.
Bounded regrets also have been shown via optimism~\cite{lattimore14bounded,guptaunified,tirinzoni2020novel}, but they are far from the AO in general.
Our novelty is to obtain instance optimality, remove the dependency on $K$, and achieve bounded regret whenever possible, simultaneously.
We make more remarks on $\ln\ln(n)$ term and how one can get rid of $\ln(|\cF|)$ and handle infinite hypothesis spaces in the appendix.

\textbf{Example: Cheating code.}
Let $\cF = \cF^{\text{code}}$ and $\sig^2=1$, and fix $f^*\in\cF$ such that $\mu^*(f^*) = 1$.
Assume $\fr{1}{2\eps} > \fr{2}{\Lam^2}$.
Then, one can show that $\gam^* = \psi(\cF) = (0,\ldots,0,\fr{2}{\Lam^2},\ldots,\fr{2}{\Lam^2})$, where the first $K_0$ coordinates are zeros, and  $\phi(\cE^*) = 0$.
We also have $|\cF| = K_0^2$, $K_\psi = \lceil \log_2(K_0)\rceil$, and $\Lam_{\min} = \eps$.
Then, using $K_0\le K$, 
\begin{align*}
\EE \Reg_n = O\del[2]{ \fr{\ln(K)}{\Lambda^2} \ln\del[2]{K n \cd \del[2]{ \fr1\eps + \fr{\ln(K)}{\Lam^2}} } + \ln(K) }~,
\end{align*}
which is $\approx\fr{\ln^2(K)}{\Lam^2} \ln(\fr n\eps)$ when taking the highest-order factors for each $(n,K,\eps,\Lam)$. We speculate that $\ln(1/\eps)$ can be removed with a tighter analysis.
We compare CROP to algorithms with AO that use forced sampling (FS in short).
Say, during $n$ rounds, FS pulls every arm $\ln\ln(n)$ times each, introducing a term $O(\eps K \ln\ln(n))$ in the regret, but let us ignore the $\ln\ln(\cd)$ factor.
For FS, the best regret bound one can hope for is $O(K\eps + \fr{\log_2(K)}{\Lam^2} \ln(n))$.
To satisfy the condition $\fr{1}{2\eps} > \fr{2}{\Lam^2}$, set $\Lam=1/2$ and $\eps=1/32$.
Then, CROP's regret is $O(\ln^2(K) \ln(n))$ whereas FS's regret is $O(K + \ln(K)\ln(n))$.
When $K\approx n$, FS has a linear regret whereas CROP has $\ln^3(n)$ regret.
If $K=2^d$ for some $d$, then the gap between the two becomes more dramatic: $O(2^d + d \ln(n))$ of FS vs $O(d^2 \ln(n))$ of CROP, an exponential gap in the nonasymptotic regime.


\section{Lower bound: Necessity of the $\Omega(\ln\ln(n))$ term}
\label{sec:lb}

One may wonder if the $\ln(\ln(n))$ term in our upper bound is necessary to achieve the asymptotic optimality up to constant factors.
We show that there exist cases where such a dependence is indeed required.
In fact, our lower bound statement is stronger; in a hypothesis class we construct for lower bound, even polynomial-regret algorithms must also pull a non-informative arm at least $\ln\ln(n)$ times.

\begin{thm}[Informal version]
  Assume the Gaussian noise model with $\sig^2 = 1$.
  There exists a hypothesis class $\cF$ and an absolute constant $C$ that satisfies the following: If an algorithm has $\EE \Reg_n \le O(n^u)$ for some $u\in \lbrack 0,1 \rparen$ under an instance $f\in\cF$, then there exists another instance $f'\in\cF$ and an arm $i$ with $\gam_i(f') = 0$ (i.e., non-informative arm) such that
  \begin{align*}
    \EE_{f'} T_i(n) = \Omega(\ln(1+(1-u)\ln(n)))
  \end{align*}
  where $\EE_{f'}$ is the expectation under the instance $f'$.
\end{thm}

The constructed instance for the lower bound is a variation of $\cH^+$ in Figure~\ref{fig:staircase}.
Our theorem shows that, just because an arm is noninformative, it does not mean that we can pull it a finite number of times. 

Our result also has an implication for algorithms with forced sampling.
To be specific, suppose that an algorithm $A$ performs forced sampling by requiring each arm to be pulled at least $\tau$ where $\tau$ is fixed at the beginning.
Then, to achieve sublinear regret bounds, it is required that $A$ use $\tau=\omega(\ln(\ln(n))$.
We emphasize that even $\tau = \Theta(\ln(\ln(n)))$ can suffer a polynomial regret, let alone being uniformly good.
This is because the constant factor matters and is a function of the problem.
We provide the precise constants, the full statement of the theorem, and its proof in our appendix.

\vspace{-.5em}
\section{Discussion}
\label{sec:discussion}
\vspace{-.5em}

There are improvements to be made including more examples and studying properties of alternative designs of $\phi$ and $\psi$, which we discuss more in the appendix.
Meanwhile, we make a few observations and open problems below.

\paragraph{It may not be the end of the optimism~\cite{lattimore17theend}.} \!\!
Let us forget about CROP and consider the oracle described in Section~\ref{sec:prelim}.
Consider $\cF^\code$  with $\fr{1}{2\eps} > \fr{2}{\Lam^2}$.
Note that UCB in fact has a regret bound of $O(\min\{\fr{K}{\eps} \ln(n), \eps n\})$; the first argument can be vacuous (i.e., $\ge n$) in which case we know the regret so far is $\eps n$ since UCB by design only pulls arm $i$ with $\Delta_i(f)=\eps$.
The oracle has regret $\Theta(\min\{\fr{\ln(K)}{\Lambda^2}\ln(n), n\})$ where we have $n$ rather than $\eps n$ because she pulls informative arms.
However, this implies that, until $n\lapp \fr{1}{\Lambda^2}$, the oracle has a linear regret.
In fact, all known algorithms with AO would be the same, to our knowledge. 
This is not just a theoretical observation.
In Hao et al.~\cite[Figure 1]{hao20adaptive}, their algorithm with AO performs worse than an optimistic one until $n \approx 2000$. 
We ask if one can achieve the minimum of the two; i.e., obtain a finite-time regret bound of $O(\min\{\fr{\ln(K)}{\Lambda^2}\ln(n), \eps n\})$.
This is a reminiscent of the ``sub-UCB'' criterion by Lattimore~\cite{lattimore18refining} (also discussed in Tirinzoni et al.~\cite{tirinzoni2020novel}) in the sense that we like to perform no worse than UCB.
For $\cF=\cF^{\code}$, we provide a positive answer in the appendix, but a more generic algorithm that enjoys the AO and performs no worse than UCB for any $\cF$ is an open problem.

\vspace{-.3em}
\paragraph{The worst-case regret.}
For more on the worst-case regret and how it is different from the instance-dependent regret, see our related work section in the appendix.
The example above shows that the oracle suffers a linear worst-case regret over the family of problems $\{\cF^{\code,\eps,\Lam}: \eps, \Lam \le \lparen0,1/2\rbrack\}$.
That is, for any given problem complexity $n$ and $K$, one can always find $\eps$ and $\Lam$ for which the oracle suffers a linear regret.
This is in stark contrast to UCB that has $\wtil O\del{\sqrt{Kn}}$ worst-case regret over this family.
In fact, the oracle suffers a linear regret in linear bandits, too.
In Example 4 of \citet{lattimore17theend}, it is easy to see that the oracle has regret $\min\cbr{2\alpha^2 \ln(n), n}$ when $\eps$ satisfies $2/\eps > 2\alpha^2$.
Thus, given $n$, this bandit problem with $\alpha \approx \sqrt{n}$ for some $\eps$ with $2/\eps > 2\alpha^2$ would make the oracle suffer a linear regret for the instance $\theta=(1,0)$.
To our knowledge, all known AO algorithms share the same trait as they do not have any device to purposely avoid pulling the informative arm in the small-$n$ regime.

We believe the issue is not that we study instance-dependent regret but that we tend to focus too much on the leading term w.r.t. $n$ in the asymptotic regime, which we attribute to the fact that it is the one where the optimality can be claimed as of now.
Less is known about the optimality on the lower order terms along with other instance-dependent parameters.
This is studied a bit more in pure exploration problems~\cite{simchowitz2017simulator,katzsamuels20anempirical}.
We hope to see more research on precise instance-dependent regret bounds in nonasymptotic regimes and practical structured bandit algorithms.

\section*{Broader Impact}

Our study is mainly about a novel approach to solve structured bandits algorithms where we try to overcome some shortcomings of existing methods.
Algorithmic developments in bandits have a huge impact in many potential applications including dose-finding trials.
In this application, a structured bandits that encode a proper inductive bias and can help resolve health issues of many people by significantly reducing time/trials needed to find dosage or the right types of drugs, leading to maximum efficacy with minimal side-effects.

\section*{Acknowledgments}

We thank the anonymous reviewers, Lalit Jain, Kevin Jamieson, Akshay Krishnamurthy, Tor Lattimore, Robert Nowak, Ardhendu Tripathy, and the organizers and participants of RL Theory Virtual Seminars for providing valuable feedback and helpful discussions.

\putbib[library-shared]
\end{bibunit}

\appendix
\addcontentsline{toc}{section}{Appendix} 

\clearpage

\part{Appendix} 

\parttoc 
\begin{bibunit}[plainnat]


\section{Related work}
\label{sec:related}

Structured bandits, generally defined, consist of bandit problems where there exist pieces of side-information or constraints in (and among) the mean rewards.
While, of course, the standard $K$-armed bandit problem~\cite{thompson33onthelikelihood} is a special case, researchers usually use the term structured bandits for nontrivial structure (e.g., beyond simple constraints like having mean rewards in $[0,1]$).
We group related work by the kinds of guarantees each study aim to achieve.

\paragraph{Worst-case regret bounds.}
Suppose we are given a predefined family of problems $\Gamma = \{\cF\}$,
The worst-case regret bound of an algorithm $\pi$ over the family $\Gamma$ is the one that answers the following question: given a set of problem complexity parameters like $n$ and $K$ (and $d$ in linear bandits, for example), what is the largest regret bound that $\pi$ can suffer over the family $\Gamma$?
As an example, for the linear bandit problem with a fixed arm set, the family $\Gamma$ contains any bandit problem $\cF$ with an arm set $\cA$ and a feature map $v:\cA \rarrow \RR^d$ for which every $f\in\cF$ and $a\in\cA$ satisfy $f(a) = \la \theta, v(a)\ra, \forall a\in\cA$ for some $\theta\in\RR^d$.

Many works in the structured bandits focus on specific parametric reward models, such as linear and generalized linear models~\cite{rusmevichientong10linearly,ay11improved,dani08stochastic,filippi10parametric}, and regret bounds of order $\sqrt{n}$ has been obtained in these settings.
Russo and Van Roy~\cite{russo13eluder} propose a general notion called eluder dimension that facilitates analysis for structured bandits for general function classes.
For nonlinear reward models, worst-case regret guarantees have been obtained on reward structures such as unimodality~\cite{yu11unimodal}, convexity~\cite{agarwal11stochastic}, and Lipschitzness~\cite{kleinberg05nearly}.

\paragraph{Instance-dependent regret bounds.}
Instance-dependent regret bounds aim to capture finer structures of problem instances beyond the complexity of hypothesis classes.
In the asymptotic regime (i.e., fixing a problem instance and letting time horizon $n$ go to infinity), many works~\cite{agrawal89asymptotically,graves1997asymptotically,lattimore17theend,hao20adaptive} derive matching asymptotic regret upper and lower bounds for \textit{uniformly-good} algorithms (defined in Section~\ref{sec:prelim}).
However, their analysis cannot be easily converted to obtain a finite-sample guarantee.

In the finite-sample regime, instance-optimal algorithms under specific model classes have been developed, such as unimodal reward~\cite{combes14unimodal} and Lipschitz rewards~\cite{magureanu2014lipschitz}.
Under general reward function classes, Combes et al.~\cite{combes17minimal} provide an asymptotically optimal algorithm that is amenable to non-asymptotic analyses, which is later extended to reinforcement learning~\cite{ok18exploration}.
However, their finite sample guarantees depend strongly on the size of the action space due to the forced sampling as discussed in Section~\ref{sec:analysis}, although such a dependence goes away in the asymptotic regime.

We remark that the forced sampling is not bad when the problem is \textit{unstructured } because the sampling complexity must scale with $K$ anyways.
For example, in the pure-exploration version of the bandit problem~\cite{bubeck09pure,bubeck11pure-tcs}, the algorithm by Garivier and Kaufmann \cite{garivier16optimal} relies on forced sampling, but there is no known evidence that the forced sampling degrades its performance while without it their algorithm can provably fail.

\paragraph{Other instance-dependent regret bounds.}
Many algorithms do not achieve the asymptotic optimality (not even within a constant factor) but possess instance-dependent guarantees in structured bandits, which is mostly based on the optimism, defined in~\eqref{eq:optimism}, or often elimination (e.g., see Jamieson~\cite[Section 6]{jamieson20informal} for linear bandits and Tirinzoni et al.~\cite{tirinzoni2020novel} for generic structure).
We remark that, unlike the standard $K$-armed bandits without structure, elimination-based approaches in structured bandits can have a different regret bound from that of the optimism (i.e., one is better than the other and vice versa depending on the problem instance) as described in Tirinzoni et al. \cite{tirinzoni2020novel}.

Many linear bandit studies obtain regret guarantees that depend on the gap $\blue\Delta$ between the best mean reward and the second-best mean reward~\cite{dani08stochastic,filippi10parametric,ay11improved}.
As another example, in Lipschitz bandits, regret guarantees that depend on the zooming dimension or near-optimality dimension have been shown~\cite{kleinberg08multi,bubeck11x}.
Other particular structures including univariate linearity~\cite{mersereau09structured}, global bandits~\cite{atan15global,shen18generalized}, regional bandits~\cite{wang18regional}, and $K$-armed bandits with side-information~\cite{bubeck13bounded,vakili13achieving,degenne18bandits} show that it is possible to achieve bounded regret.

For the generic structures, many studies provide finite-time instance-dependent guarantees that do not achieve a constant-factor asymptotically optimality~\cite{lattimore2014bounded,azar13sequential,tirinzoni2020novel,guptaunified}, except for known special cases where the reward function class has a factorized representation across different arms, i.e., in the settings of Lai and Robbins~\cite{lai85asymptotically}, and Bernetas and Katehakis~\cite{burnetas1996optimal}.
These studies, however, have regret bounds that reflect the structure of the instance beyond the gap $\Delta$ mentioned above.
We conjecture that the suboptimality in their bounds is rooted in their confidence sets that are an intersection of confidence intervals for each arm, let alone their sampling strategies.
In contrast, CROP maintains a confidence set that captures the structure of the hypothesis class, which, as we show in the proof, allows us to connect the constraints of $c(f)$ in~\eqref{eqn:gamma} to the concentration inequality and thus to the confidence set as well.

Although we focus on the finite hypothesis space for simplicity, our ultimate goal of the paper is to find a generic algorithmic principle for any hypothesis space.
Our algorithm CROP achieves the asymptotic optimal regret with a constant-factor, enjoys bounded regret whenever possible, and has mild dependence on $K$, all thanks to our novel forced-sampling-free design.
This partially resolves the open question raised by Tirinzoni et al. \cite[Section 7]{tirinzoni2020novel} where they ask if one can design confidence-based strategies (as opposed to solving the optimization problem~\eqref{eqn:gamma} along with forced sampling) that are optimal for general structures with good finite-time performance.
We believe that such an advancement is not an artifact of the finite hypothesis space setting but the fact that we rely on the pessimism, the key novelty in our algorithm design.

\section{Details of the arguments}

In this section, we provide more explanations on our argument in the main body in the order they appear, section by section.

\subsection{More on problem definition and preliminaries}
\label{subsec:more-prelim}

Let us elaborate more on the cheating code example in Section~\ref{sec:prelim}.

\textbf{The cheating code example.}
Let us elaborate more on the cheating code example in Section~\ref{sec:prelim}.
To understand the optimization problem $c(f)$ for the cheating code class, let us consider a simple case where $K_0=2$ and $k=1$.
Without loss of generality, let $f^* = (1, 1-\eps, 0)$.
In this case, the only competing hypothesis is $(1, 1+\eps, \Lam)$.
Then, $c(f^*)$ is written as follows:
\begin{align*}
\begin{array}{rrl}
c(f^*) =   \displaystyle\min_{\gam_1=0, \gam_2\ge0, \gam_3 \ge0}
& \multicolumn{2}{l}{
  \eps\gam_2 + 1\cd \gam_3
}
\\ \mbox{s.t.}
&   \gam_2 \fr{(2\eps)^2}{2} + \gam_3 \fr{\Lam^2}{2} &\ge 1
\end{array}
\end{align*}
One can see that a solution $\gam_2 =\alpha\cd \fr{1}{(2\eps)^2/2}$ and $\gam_3 = (1-\alpha)\cd \fr{1}{\Lam^2/2}$ for some $\alpha\in[0,1]$ is feasible and satisfies the constraint with equality.
Furthermore, one can see that any feasible solution that cannot be expressed by the solution above has a strictly larger objective value due to $\gam\ge0$ and $\eps>0$.
The objective function is then $\alpha\cd \fr{1}{2\eps} + (1-\alpha) \cd \fr{2}{\Lam^2}$.
This means that whenever $\fr{1}{2\eps} > \fr{2}{\Lam^2}$, setting $\alpha=0$ achieves the minimum.

For generic $K_0$, we first provide an example of $f^* = (1, 1-\eps, 1-\eps, 1-\eps, 0, 0)$ for the case of $K_0 = 4$.
\begin{align}\label{eq:opt-ex}
\begin{aligned}
\begin{array}{rllllrl}
c(f^*) = \displaystyle\min_{\gam_1=0, \gam_{2:K}\in\lparen 0,\infty\rbrack^{K-1}}
& \multicolumn{4}{l}{
  \eps\cd \del{\sum_{a=2}^{K_0} \gam_a} + 1\cd\del{\sum_{a=K_0+1}^K \gam_a}
}
\\ \mbox{s.t.}
&\gam_2 \fr{(2\eps)^2}{2} & & & &+ \gam_6 \fr{\Lam^2}{2}  &\ge 1
\\   & &+ \gam_3 \fr{(2\eps)^2}{2} &  & + \gam_5 \fr{\Lam^2}{2} & &\ge 1
\\   &  & &+ \gam_4 \fr{(2\eps)^2}{2} &  + \gam_5 \fr{\Lam^2}{2} &+ \gam_6 \fr{\Lam^2}{2}  &\ge 1
\end{array}
\end{aligned}
\end{align}
We use shorthands $\gam_{i:j}$ for $\gam_i, \gam_{i+1}, \ldots, \gam_j$.
It is now nontrivial to see how the optimal solution would look like.
The following proposition provides a characterization of the optimal solution.
\begin{prop}\label{prop:opt-solution}
  Consider $\cF^\code$.
  Let $K_0 = 2^k$ for some integer $k$.
  We claim that if $\fr{1}{2\eps} > \fr{2}{\Lam^2}$, then the solution of the optimization problem~\eqref{eqn:gamma} is
  \begin{equation}
  \gam^{\normalfont\dagger}_a =
  \begin{cases}
  0, & a \in [K_0] \\
  \frac{2}{\Lambda^2}, & a \in \cbr{ K_0 + 1, \ldots, K}
  \end{cases}~.
  \label{eqn:opt-code}
  \end{equation}
\end{prop}
\begin{proof}
  For clarity, our convention is that the coordinate $K$ is for the least significant bit of the code.
  For example, when $K_0 = 4$, then the hypothesis $f$ with $a^*(f) = 2$ is $f = (1,1+\eps,1-\eps,1-\eps,0,\Lam)$.
  
  The plan is to suppose that $u$ is a feasible solution to the optimization problem $c(f^*)$ for $\cF^\code$ for which there exists a coordinate $a\in[K_0] \sm \cbr{a^*}$ with $u_a > 0$.
  Then, we show:
  \begin{itemize}
    \item First, we prove that it is possible to construct a feasible solution $v$ that is strictly better than $u$ where $v$ does not have nonzero entries for the first $K_0$ coordinates; this proves that the optimal solution must be supported only on the cheating arms.
    \item Second, we show that $\gam^\dagger$ is the optimal solution.
  \end{itemize}
  
  As a starter, consider the example of~\eqref{eq:opt-ex}.
  Suppose we have a feasible solution $u=(0, \fr{1}{4{\eps}^2}, 0, 0, \fr{2}{\Lam^2}, \fr{1}{\Lam^2})$.
  Then, the coordinate $q = 2$ is nonzero.
  Consider a hypothesis $g$ that has arm $q$ as the best arm: $g = (1, 1+\eps, 1-\eps, 1-\eps, 0, \Lam)$.
  The arm $6$ is the cheating arm whose mean reward is different from that of $f^*$; let $j=2$ so that $6=K_0+j$.
  Then, we can modify $u$ by zeroing out $\gam_q$ and adding more mass to $\gam_{K_0+j}$ so that the first constraint of~\eqref{eq:opt-ex} is satisfied.
  This modification leads to $h^* = (0, 0, 0, 0, \fr{2}{\Lam^2}, \fr{2}{\Lam^2})$. 
  Note that this operation does not make other constraints violated because the variable $\gam_q$ appears in one constraint only, and adding more mass to $\gam_{K_0+j}$ never harms.
  We now generalize this example.
  
  Let $q \in [K_0] \sm\cbr{a^*}$ satisfy $u_q > 0$.
  Consider $g\in\cC^*$ with $a^*(g) = q$.
  Let ${j} \in[k]$ be the largest index $j$ for which $g(K_0+j) \neq f^*(K_0 + j)$ whose existence is certified by the definition of $\cF^\code$.
  Let $e_i$ be the $i$-th indicator vector.
  Let $\delta = \fr{2\eps^2 u_q}{\Lam^2/2}$.
  We then define
  \begin{align*}
  \blue{h^*(u,q)} = u - u_q e_q + \delta\cd e_{K_0 + j }
  \end{align*}
  One can show that our choice of $\delta$ indeed ensures that $h^*(u,q)$ is a feasible solution.
  We now show that this modified version $h^* := h^*(u,q)$ has a strictly smaller objective function value:
  \begin{align*}
  \del{\sum_{a=1}^{K_0}\eps u_a + \sum_{a=K_0+1}^{K} \Delta_a u_a} - \del{\sum_{a=1}^{K_0}\eps  h^*_a + \sum_{a=K_0+1}^{K} \Delta_a  h^*_a}
  &= \eps u_q  - \Delta_{j } \dt
  \\&\ge \eps u_q  - \dt
  \tag*{($\because~ \Delta_{j} \le 1$)}
  \\&= \eps u_q  - \fr{2\eps^2 u_q}{\Lam^2/2}
  \\&= \eps u_q(1  - \fr{2\eps}{\Lam^2/2})
  \\&> 0 \tag*{($\because$~ the assumption of the proposition)}
  \end{align*}
  Therefore, one can perform the following one sweep of coordinate descent.
  \kjunbox{
    \begin{itemize}
      \item Input: $f^*$, $u$: a feasible solution of $c(f^*)$.
      \item $v^{(1)} \leftarrow u$
      \item For $a = 1,\ldots, K_0$,
      \begin{itemize}
        \item $v^{(a+1)} \leftarrow h^*(v^{(a)},a)$
      \end{itemize}
      \item Output: $v^* := v^{(K_0+1)}$
    \end{itemize}
  }
  We conclude the first part of the proof by the following observations:
  \begin{itemize}
    \item The output $v^*$ above has a strictly smaller objective function than that of the input $u$ whenever $\exists a\in[K_0]: u_a > 0$.
    \item The output $v^*$ satisfies that $v^*_{1:K_0} = 0$.
  \end{itemize}
  
  We now show the second part of the proof.
  Let $\gam$ be a feasible solution that is supported only on the cheating arms.
  We claim that $\forall j \in [k]$, the value $\gam_{K_0+j}$ must be at least $\frac{2}{\Lambda^2}$.
  
  The reason is that, given $j\in[k]$, we can find $g\in\cC^*$ for which $(g(K_0+1),\ldots, g(K))$ differs from $(f^*(K_0+1),\ldots,f^*(K))$ at the coordinate $j$ only, by the definition of $\cF^\code$; i.e., the binary representation of $a^*(g)-1$ differs from that of $a^*-1$ only at the $(k-j+1)$-th least significant bit.
  In the example of~\eqref{eq:opt-ex}, for $j=1$, $g=(1,1+\eps,1-\eps,1-\eps,0,\Lam)$ and for $j=2$, $g=(1,1-\eps,1+\eps,1-\eps,\Lam,0)$.
  Then, the constraint induced by that hypothesis $g$ is:
  \[
  \gamma_{a^*(g)} \frac{(2\eps)^2}{2} + \gamma_{K_0+j} \frac{\Lambda^2}{2} \geq 1.
  \]
  Plugging the fact that $\gamma_{a^*(g)} = 0$,  we have $\gamma_{K_0+j} \frac{\Lambda^2}{2} \geq 1$, implying that $\gamma_{K_0+j} \geq \frac{2}{\Lambda^2}$.
  This proves the claim above, establishing a coordinate-wise lower bound on $\gam$.
  
  We observe that the objective function of this lower bound on $\gam$ is a lower bound for the optimal solution, which is achieved by $\gam^\dagger$; this concludes the proof.
\end{proof}

\subsection{More on CROP}

We here discuss more on the design choices for CROP.
We do not claim our design of $\psi$ and $\phi$ is the best; fine-tuning of those and the algorithm itself is left as future work.
In what follows, we focus on describing our intention on the current design \textit{at the time of development}.

\textbf{The design of $\psi(f)$ in \fallback.}
The reason why we now allow $\gam_{a^*(f)} \neq 0$ is for the following case where $\eps>0$ and we boldface the best arms:
\begin{center}\begin{tabular}{C{5ex}ccccc}\hline
    Arms & A1 & A2 & A3 \\ \hline
    $f_1$ &  \textbf{1}   & .25 & .25 \\
    $f_2$ &  \textbf{.75} & .25 & .25 \\
    $f_3$ &  0   & .25~--~2$\eps$ & \textbf{.25~--~$\eps$} \\
    $f_4$ &  0   & \textbf{.25} & 0 \\ \hline
  \end{tabular}
\end{center}
Suppose $f^* = f_2$ and $\eps>0$ is small enough.
At the beginning, $\cF_t = \{f_1,f_2,f_3,f_4\}$, $\tilcF_t = \{f_1\}$, $\barf_t = \{f_3\}$.
If we do not allow $\gam_{a^*(f)}$ to be nonzero in~\eqref{eq:psi}, then $\psi(\barf_t)$ will assign a very large number of pulls to arm A2 whereas pulling A3 will eliminate $f_4$ quickly.
On the other hand, those troublesome hypotheses $f_3$ and $f_4$ belong to the docile class w.r.t. $f_2$, and incurring high regret for those (albeit finite terms) seems unreasonable.

We have a constraint $\gam \succeq \phi(f) \vee \gam(f)$ in~\eqref{eq:psi}, which we call the \textit{extension} constraint.
The intention is, if $f$ is the ground truth, we will have to make $\gam_a(f) \ln(n)$ pulls for the informative arms $\{a: \gam_a(f) > 0\}$ anyways, so we add it to the constraint.
This guards against to the case where we inadvertently pull noninformative arms $\{a: \gam_a(f) = 0\}$ too much.
Our understanding is that, in general, this does not affect the regret bound too much, at least in the current analysis.

\subsection{More on analysis}

The proof of Theorem~\ref{thm:main} can be found in Appendix~\ref{sec:proofs}.

\begin{remark}
  Note we have a $\ln(|\cF|)$ dependence, which comes from the naive union bounds.
  One can extends CROP to a larger or even infinite hypothesis space using the covering number argument as done in Foster et al.~\cite[Lemma 4]{foster18practical}, although we have focused on the finite hypothesis for simplicity.
\end{remark}

\subsection{More on discussion}
\label{sec:more-discussion}
We provide missing details from the discussion section.

\textbf{Regret no more than the optimism.}
Consider $\cF^{\text{code}}$.
Assume $f^* = (1,1-\eps,\ldots,1-\eps,0, \ldots, 0) \in \cF^\code$ without loss of generality, and assume $\fr{1}{2\eps} > \fr{2}{\Lam^2}$.
The goal is to achieve a regret bound of $\EE \Reg_n = O\del{\min\cbr{ \fr{\ln(K)}{\Lam^2}\ln(n), \eps n}}$, which is depicted in Figure~\ref{fig:regretnomore}.
\begin{figure}
  \centering
  \includegraphics[width=.35\linewidth,valign=t]{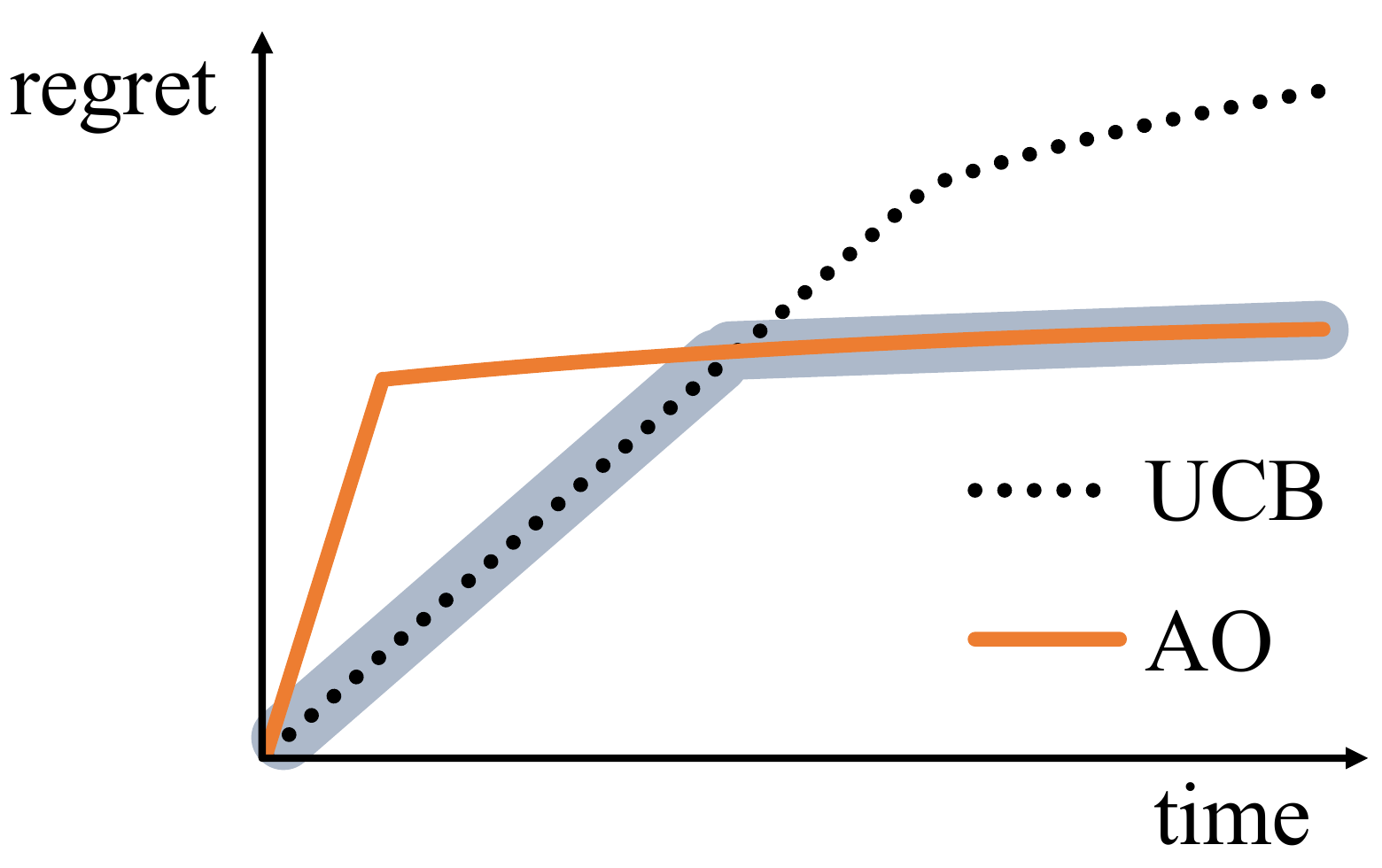}
  \caption{A cartoon showing the regret bound of UCB and AO where AO is the asymptotically optimal oracle described in Section~\ref{sec:prelim}. The minimum of the two curves is shaded. Can we achieve this regret bound? }
  \label{fig:regretnomore}
\end{figure}
Note that, in the fixed budget setting (i.e., the target $n$ is given before running the algorithm), it is trivial to perform no worse than the optimism.
Specifically, check the regret bound of the oracle or those that mimic the oracle (simply call them the oracle, hereafter) and see if it is larger than $\eps n$.
If true, then simply pull any of the first $K_0$ arms uniformly throughout; otherwise, invoke the oracle.

For achieving the goal in the anytime setting, note that the optimism, such as UCB1~\cite{auer02finite} run without knowing the structure or UCB-S~\cite{lattimore14bounded}, achieves an anytime regret bound of
\begin{align*}
\EE \Reg_n \le \min\cbr{ c_1 \fr{K}{\eps} \ln(n),~ \eps n}
\end{align*}
for $n\ge2$ and some numerical constant $c_1$.
The oracle achieves an anytime regret bound of
\begin{align*}
\EE \Reg_n \le  \min\cbr{ c_2 \fr{\ln(K)}{\Lam^2} \ln(n),~ n}
\end{align*}
for $n\ge2$ and some numerical constant $c_2$.

Suppose $K$ is large enough so that $c_1 \fr{K}{\eps} > c_2 \fr{\ln(K)}{\Lam^2}$.
Let $\blue{t_0}$ be the $t$ such that $c_2\fr{\ln(K)}{\Lam^2}\ln(t) = \eps t$, the time step after which the oracle outperforms UCB.
The idea is simple: run UCB up to time step $t \le t_0$ and then start the oracle as if we are starting from the beginning (in words, throw away all the samples so far).
Then,
\begin{align*}
\EE \Reg_n &= \one\{n \le t_0\} \eps n + \one\{n > t_0\} \EE\sbr{{ \eps t_0 + c_2 \fr{\ln(K)}{\Lam^2}\ln(n - t_0)}}
\\&\le  \one\{n\le t_0\} 2\cd \eps n
+ \one\{n>t_0\} 2\cd c_2\fr{\ln(K)}{\Lam^2} \ln (n)
\\&\le 2\cd \min\cbr{c_2 \fr{\ln(K)}{\Lam^2}\ln(n), \eps n}~,
\end{align*}
which achieves our goal.

Note, however, that this simple strategy was possible because we knew $\eps$ for the given class $\cF^\code$.
Consider $\blue{\cF^{\text{code2}}} := [1/2,1]^{K_0} \times \{0,\Lam\}^{k}$ where the last $k=\lceil\log_2(K_0)\rceil$ arms are binary codes encoding the best arm index (break ties with the smallest index).
Then, the class contains hypotheses that have $\eps$ gap for the first $K_0$ arms for any $\eps$.
Even if we do know that $f^*$ is one of those cases, we do not know $\eps$ and thus have to adaptively decide when to start pulling the informative arms via the rewards collected throughout.
Achieving the asymptotic optimality while maintaining the standard worst-case regret (e.g., $\sqrt{dn\ln(K)}$ for linear bandits) seems to be an interesting open problem.

\textbf{Future work.}
Our study unlocks numerous open problems besides achieving both the worst-case regret and asymptotic optimality.
Since many bandit studies have focused on the leading term $\ln(n)$, the optimal sampling strategy that minimize the arm pulls of those noninformative arms is not studied well.
We like to study the optimality of the number of arm pulls of noninformative arms w.r.t. not just $n$ but also problem-dependence parameters such as information gaps, which tend to matter when the number of arms is very large.
Towards practical algorithms, we believe this is more important than getting the exact asymptotic optimality.
Furthermore, we like to investigate if we can extend our pessimism to more popular structures such as linear bandits or Lipschitz bandits so one can achieve the asymptotic optimality without forced sampling (and without the dependence on the number of arms).


\section{Upper bound proof}
\label{sec:proofs}

In this section, we discuss the proof of Theorems~\ref{thm:main} and auxiliary lemmas for it.

\begin{thm}  \label{thm:main-appendix}
  Let $(\alpha,\ralpha,z,\rz) = (2,3,|\cF|,|\cF|)$.
  Suppose we run \crop with hypothesis class $\cF$ with the environment $f^*\in\cF$.
  Then, \crop has the following anytime regret guarantee: $\forall n\ge2$,
  \begin{align*}
  \EE\Reg_n
  &\le c_1 \cd \del{  P_1 \ln(n) +  P_2 \ln(\ln(n))
    +  P_3  \del{\ln(|\cF|)  +  \ln\del{ Q_1 }}
    + K_\psi \Delta_{\max} }~,
  \end{align*}
  where $c_1$ is a numerical constant, and
  \begin{align*}
  P_1 = \sum_a \Delta_a \gam^*_a,
  ~ P_2 = \sum_{a} \Delta_a \phi_a(\cE^*),
  ~ P_3 = \sum_a \Delta_a \psi_a({\cF})
  \end{align*}
  and $\blue{Q_1}= \Lam_{\min}^{-2} + K_\psi(1+  \max_a \psi_a({\cF}))$.
  Furthermore, when $\gam^* = 0$, we have $P_1=P_2=0$, achieving a bounded regret.
\end{thm}

Before going into details we highlight some of the technical aspects of our proof.
In our proof, \Cref{lem:boundarmpull} plays a key role for analyzing the confidence set $\rcF_t$ that has an aggressive confidence level, which we believe is not commonly dealt with in the standard bandit settings.
One can see why this is needed in Section~\ref{subsec:ztb} and Section~\ref{subsec:ztc} where we apply the ``regret peeling.''
The proofs in Section~\ref{subsec:zo} are a bit lengthy, but the main idea stems from Lattimore and Munos~\cite{lattimore14bounded}.
Other proofs are relatively standard, we believe.

We first present some martingale concentration inequalities for our confidence set.

\subsection{Concentration inequalities}
This section establishes a few important concentration inequalities on the losses of reward regressors, which are instrumental in our analysis.

\paragraph{Additional notations.} Recall that at time step $t$, the learner pulls arm $a_t$ and receives reward $r_t = f_*(a_t) + \blue{\xi_t}$, where $\xi_t$ is $\sig^2$-sub-Gaussian.
To avoid double subscripting, we will sometimes use $a(t)$ to denote $a_t$.
Define the instantaneous loss of regressor $f$ as time $s$ as $\blue{\ell_s(f)} := (f(a_s) - r_s)^2$.
At the end of time step $t$, we define the cumulative loss of $f$ as $\blue{L_t(f)} := \sum_{s=1}^t \ell_s(f)$.
Define the instantaneous regret of $f$ as $\blue{M_s(f) }:= \ell_s(f) - \ell_s(f^*)$.
We define the \textit{information} gap between $f$ and $f^*$ on action $a$ as: $\blue{\Lam_a(f)} := \fr{ f(a(s)) - f^*(a(s)) }{\sig}$.
A larger $\Lam_a(f)$ implies that $f$ is easier to be distinguished from $f^*$ by pulling arm $a$.
This should not to be confused with the \textit{reward} gap $\Delta_a(f) = \mu^*(f) - f_a$.

We define the filtration $\cbr{\Sigma_t}_ {t=0}^\infty$ as follows: $\blue{\Sigma_t} = \sig( a_1, r_1, \ldots, a_{t}, r_{t}, a_{t+1} )$. We abbreviate $\EE_t\sbr{\cdot } = \EE\sbr{\cdot \mid \Sigma_t}$.
Define $\blue{\IC(f,g,\pi)} = \sum_{a \in \cA} \pi(a) \fr{(f(a) - g(a))^2}{2\sig^2}$ as the \blue{information constraint} between $f$ and $g$ w.r.t. $\pi$.
Define $\blue{\IC^*(f,\pi)} = \IC(f,f^*,\pi)$.

Define $\blue{T(t)} = \del{T_a(t)}_{a \in \cA}$ as the vector that encodes the number of arm pulls for each arm up to time step $t$.
For vectors $u = (u_a)_{a \in \cA}$, $v = (v_a)_{a \in \cA}$, we denote $\blue{u \succeq v}$ if $\forall a \in \cA, u_a \geq v_a$.
We use shorthands $\blue{a^*} = a^*(f^*)$ and $\blue{\mu^*} = \mu^*(f^*)$.

We establish fundamental concentration result on partial sums of $\{M_s(f)\}$.

\begin{lem}  \label{lem:conc-main}
  \begin{enumerate}[(i)]
    \item For any $r > 0, \beta > 0$,
    \[ \PP \del{  \exists t \in \NN \centerdot \IC(f, f^*, T(t)) \geq r  ~\wedge~  L_t(f) - L_t(f^*) \leq  \beta } \leq \exp\del{-\frac{\sig^2 r - \beta}{4\sig^2}}.  \]
    \label{item:ic-elim}
    \item For any $\beta > 0$,
    \[ \PP \del{ \exists t \in \NN \centerdot  L_t(f^*) - L_t(f) \geq  \beta } \leq \exp\del{-\frac{\beta}{4\sig^2}}. \]
    \label{item:fstar-in}
  \end{enumerate}
\end{lem}
\begin{proof}
  Throughout the proof, we will abbreviate $M_s(f)$ as $M_s$ to avoid notation clutter.
  It can be easily seen that $\sum_{s=1}^t M_s = L_t(f) - L_t(f^*)$.
  Define $\blue{d_s} = \sig\Lam_{a(s)}(f) = f(a) - f^*(a)$.
  
  By Equation~\eqref{eqn:sq-lb} in Lemma~\ref{lem:sq-bounds} below with $\lambda = \frac{1}{4\sigma^2}$ and $\delta = \exp\del{-\frac{\beta'}{4\sigma^2}}$, we have  
  \begin{equation}
  \PP \del{ \exists t \in \NN \centerdot \frac12\sum_{s=1}^t d_s^2 - \sum_{s=1}^t  M_s \geq \beta'  } \leq \exp\del{-\frac{\beta'}{4\sig^2}}.
  \label{eqn:disag-excess}
  \end{equation}

  This immediately implies the two items; for the first item,
  using $\IC(f,f^*,T(t)) = \fr12\sum_{s=1}^t {\fr{d_s^2}{\sig^2}}$, we have
  \[ \IC(f, f^*, T(t)) \geq r  ~\wedge~   \sum_{s=1}^t  M_s \leq  \beta \implies \fr{1}{2}\sum_{s=1}^t  d_s^2 - \sum_{s=1}^t  M_s \geq \sig^2 r - \beta. \]
  for the second item, observe that
  \[ \sum_{s=1}^t  M_s \leq  -\beta \implies  \fr{1}{2}\sum_{s=1}^t d_s^2 - \sum_{s=1}^t  M_s \geq \beta. \qedhere \]
  \end{proof}

\begin{lem}
Suppose $f$ is in $\cF$. Define $\blue{d_s(f)} = \sig\Lam_{a(s)}(f) = f(a) - f^*(a)$. Then, for any $\lambda > 0$ and $\delta > 0$,
\begin{equation}
  \PP\del{ \exists t \in \NN \centerdot \sum_{s=1}^t M_s(f) \geq (1 + 2\sigma^2 \lambda) \sum_{s=1}^t d_s(f)^2 + \frac 1 \lambda \ln\frac 1 \delta } \leq \delta,
\label{eqn:sq-ub}
\end{equation}
\begin{equation}
  \PP\del{ \exists t \in \NN \centerdot \sum_{s=1}^t M_s(f) \leq (1 - 2\sigma^2 \lambda) \sum_{s=1}^t d_s(f)^2 - \frac 1 \lambda \ln\frac 1 \delta } \leq \delta.
\label{eqn:sq-lb}
\end{equation}
\label{lem:sq-bounds}
\end{lem}
\begin{proof}
Throughout the proof, we will abbreviate $M_s(f)$ as $M_s$, and abbreviate $d_s(f)$ as $d_s$, to avoid notation clutter.

 First, observe that $\ell_s(f^*) = \xi_s^2$. in addition,
  \begin{align*}
  M_s = \ell_s(f) - \ell_s(f^*)
  = (d_s  - \xi_s)^2 - \xi_s^2
  = d_s^2  - 2 d_s \xi_s,
  \end{align*}
  which implies that
  \[ \EE_{s-1} \sbr{ M_s } =  d_s^2.  \]
  
For any $\theta \in \RR$, define $\blue{H_t} = \exp\del{\theta \sum_{s=1}^t M_s - \theta(1 + 2\sigma^2\theta) \sum_{s=1}^t d_s^2)}$ with the convention $H_0 = 1$. We now show $\cbr{H_t}_{t=0}^n$ is a nonnegative supermartingale.

By the calculations on $M_s$ above, $\theta M_s - \theta(1 + 2\sigma^2\theta) d_s^2 = -2 \theta d_s \xi_s - 2\sigma^2 \theta^2 d_s^2$.
Using this, we have
  \[ \EE_{t-1} \sbr{H_t} = \EE_{t-1} \sbr{ H_{t-1} \cdot \exp\del{\theta M_t - \theta(1 + 2\sigma^2\theta) d_t^2}  } = H_{t-1} \cdot \EE_{t-1} \sbr{ \exp\del{-2 \theta d_t \xi_t - 2\sigma^2 \theta^2 d_t^2}  }. \]

Observe that by the $\sigma^2$-subgaussian property of $\xi_t$,
\begin{align*}
\EE_{t-1} \sbr{ \exp\del{-2 \theta d_t \xi_t - 2\sigma^2 \theta^2 d_t^2}  }
& =
\EE_{t-1} \sbr{ \exp\del{-2 \theta d_t \xi_s}  } \exp\del{- 2\sigma^2 \theta^2 d_t^2} \\
& \leq
\exp\del{ \frac{(2 \theta)^2 d_t^2 \sigma^2}{2} } \exp\del{- 2\sigma^2 \theta^2 d_t^2} = 1.
\end{align*}
This shows that $\EE_{t-1} \sbr{H_t} \le H_{t-1}$, proving $\cbr{H_t}_{t=0}^n$ is a nonnegative supermartingale.
  Therefore, by Ville's maximal inequality~\cite{ville39etude} (see also~\citet[Exercise 5.7.1]{durrett10probability}), we have that for any $\delta > 0$, 
\[
\PP( \exists t \in \NN \centerdot H_t \geq \delta^{-1}) \leq \delta.
\]
Choosing $\theta = \lambda$, and noting that $H_t \geq \frac 1 \delta$ is equivalent to $\sum_{s=1}^t M_s \geq (1 + 2\sigma^2 \lambda) \sum_{s=1}^t d_s^2 + \frac 1 \lambda \ln\frac 1 \delta$, we get Equation~\eqref{eqn:sq-ub}. Likewise, choose $\theta = -\lambda$, and noting that $H_t \geq \frac 1 \delta$ is equivalent to $\sum_{s=1}^t M_s \leq (1 - 2\sigma^2 \lambda) \sum_{s=1}^t d_s^2 - \frac 1 \lambda \ln\frac 1 \delta$, we get Equation~\eqref{eqn:sq-lb}.
\end{proof}

\begin{lem} \label{lem:concentration-fstar-anytime}
  Let $\beta_t = 4\sig^2 \ln(zt^\alpha)$ with $z \ge |\cF|$ and $\alpha \ge 1$.
  Define $B_t = \{L_{t-1}(f^*) - \min_{f\in\cF} L_{t-1}(f) > \beta_t\}$.
  Let $q\ge1$ be an integer.
  Then,
  \begin{align*}
  \PP(\exists t \ge q: B_t) \le \lt(\fr{1}{q}\rt)^\alpha.
  \end{align*}
\end{lem}
\begin{proof}
  \begin{align*}
  \PP(\exists t \ge q: B_t)
  &=
  \PP(\exists t \ge q: L_{t-1}(f^*) - \min_{f \in \cF} L_{t-1}(f) > \beta_t) \\
  &\leq
  \PP(\exists t \ge q: L_{t-1}(f^*) - \min_{f \in \cF} L_{t-1}(f) > \beta_q) \\
  &\leq
  \PP(\exists t \in \NN, f \in \cF: L_{t-1}(f^*) - L_{t-1}(f) > \beta_q) \\
  &\leq
  \sum_{f \in \cF} \PP(\exists t \in \NN: L_{t-1}(f^*) - L_{t-1}(f) > \beta_q) \\
  &\leq
  \abs{\cF} \cdot \exp\del{-\frac{\beta_q}{4\sig^2}}
  \leq
  \lt(\fr{1}{q}\rt)^\alpha.
  \end{align*}
  where the first inequality is from the fact that $\beta_t$ is monotonically increasing; the second inequality
  is by relaxing the range of $t$; the third inequality is from union bound; the fourth inequality is from item~\ref{item:fstar-in}
  of Lemma~\ref{lem:conc-main}; the last inequality is by algebra.
\end{proof}

\subsection{Generic lemmas for analyzing bandit algorithms}
The following is a standard inequality used in the UCB analysis.
\begin{lem}
  \label{lem:boundarmpullzero}
  Consider any bandit algorithm. Let $a_t$ be the index of the arm pulled at time $t$ and $T_i(t-1)$ be the number of times arm $i$ is pulled up to (and including) time $t-1$.
  Let $\tau$ be an integer and $Q_t$ be an event. Then,
  \begin{align*}
  \sum_{t=1}^n \one\{a_t = i, Q_t\}
  \le \tau +   \sum_{t=\tau+1}^n \one\{a_t = i, Q_t, T_i(t-1)\ge\tau\}
  \end{align*}
\end{lem}
\begin{proof}
  \begin{align*}
  \sum_{t=1}^n \one\{a_t = i, Q_t\}
  = \sum_{t=1}^n \one\{a_t = i, Q_t, T_i(t-1)< \tau\} + \sum_{t=1}^n \one\{a_t = i, Q_t, T_i(t-1) \ge \tau\}
  \end{align*}
  The first summation is bounded by $\tau$, for the following reason: if there are $\tau+1$ time steps $t_1 <\ldots < t_{\tau+1}$ in which $a_t = i, Q_t, T_i(t-1)< \tau$ holds, we have that $T_i(t_{\tau+1}-1) \geq \tau$, which contradicts with the fact that $T_i(t_{\tau+1}-1)< \tau$.
  
  Furthermore, for $t \leq \tau$, it must be the case that $T_i(t-1) \leq t -1 \leq \tau - 1$, therefore the first $\tau$ terms of the second sum must be zero. This implies that the second sum equals $\sum_{t=\tau+1}^n \one\{a_t = i, Q_t, T_i(t-1) \ge \tau\}$.
\end{proof}

In fact, the indicator terms in the lemma above have dependencies between different time steps that is not being explicitly captured.
When we take the expectation and apply concentration inequalities, these dependencies are lost.
The following lemma extends \Cref{lem:boundarmpullzero} so that such a dependency becomes explicit, which help prove tighter bounds.
The basic idea is that whenever we pull an arm $a$ at time $t$, the count of arm $a$ increases by 1, so by the time we pull arm $a$, the pull count must be larger; this helps, when taking the expectation, obtain a tighter concentration of measure.

\begin{lem} \phantom{;}\label{lem:boundarmpull} 
  Under the same assumptions in \Cref{lem:boundarmpullzero},
  \begin{align*}
  \sum_{t=1}^n \one\{a_t = i, Q_t\}
  \le \tau + \sum_{m=1}^\infty \one\lt\{\exists t\ge\tau+1: a_t = i, T_i(t-1)\ge \tau + m - 1,  Q_t\rt\}
  \end{align*}
\end{lem}
\begin{proof}
  By \Cref{lem:boundarmpullzero},
  \begin{align*}
  \sum_{t=1}^n \one\{a_t = i, Q_t\}
  \le \tau +   \sum_{t=\tau+1}^n \one\{a_t = i, Q_t, T_i(t-1)\ge\tau\}
  \end{align*}
  Define $\blue{t^-} = t-1$
  Define an event:
  \begin{align*}
  \blue{A_t} = \lt\{a_t = i, T_i(t^-) \ge \tau, Q_t\rt\}
  \end{align*}
  We aim to bound $\sum_{t=\tau+1}^n \one\{A_t\}$.
  Define $\blue{t_m}$ to be the $m$-th time step after (and including) $t=\tau+1$ that $A_t = 1$ is true; i.e.,
  \begin{align*}
  t_1 &:= \min\{t \in [\tau+1, n]: A_t \text{ is true }\}
  \\\forall m\ge2, t_m &:= \min\{t \in [t_{m-1}+1, n]: A_t \text{ is true } \}
  \end{align*}
  where we take $\min \emptyset = \infty$.
  One can verify that, if $t_m <\infty$, then $T_i(t_m^-) \ge \tau + m -1$.
  Then,
  \begin{align*}
  \sum_{t=\tau+1}^n \one\{A_t\}
  &= \sum_{m=1}^\infty \one\{t_m < \infty\}\one\{A_{t_m}\}
  \\&= \sum_{m=1}^\infty \one\lt\{t_m < \infty, a_{t_m} = i, T_i(t_m^-) \ge \tau, Q_{t_m}\rt\}
  \\&\sr{(a)}{\le} \sum_{m=1}^\infty \one\lt\{t_m < \infty, a_{t_m} = i, T_i(t_m^-) \ge \tau + m - 1,  Q_{t_m}\rt\}
  \\&\le \sum_{m=1}^\infty \one\lt\{t_m < \infty, \exists t\ge\tau+1: a_t = i, T_i(t^-) \ge \tau + m - 1,  Q_t\rt\}
  \\&\le \sum_{m=1}^\infty \one\lt\{\exists t\ge\tau+1: a_t = i, T_i(t^-) \ge \tau + m - 1,  Q_t\rt\}
  \end{align*}
  where $(a)$ is by our observation above.
\end{proof}

\subsection{Lemmas on the execution of \crop}

Depending on the execution trace of \crop, at time $t$, we define four events that form a disjoint union of the sample space:
\begin{enumerate}
  \item `E'xploit: $\blue{\Ex_t}$, i.e. $a^*(\cF_t)$ is singleton (line~\ref{line:exploit})
  \item `C'on`f'lict: $\blue{\Cf_t}$ (line~\ref{line:cf})
  \item `F'ea`s'ibility: $\blue{\Fs_t}$ (line~\ref{line:fs})
  \item `F'all`b'ack: $\blue{\Fb_t}$ (line~\ref{line:fb})
\end{enumerate}

The following lemma becomes useful when showing that, even if one utilizes $f$ to eliminate $g$ while neither $f$ nor $g$ is the ground truth, she will successfully eliminate either $f$ or $g$, up to a constant-factor w.r.t. the arm pulls.

\begin{lem}
  For any $f, g$ and $\pi \succeq 0$, we have
  $\IC(f,g,\pi) \leq \max( \IC^*(f,4\pi ), \IC^*(g,4\pi) )$.
  \label{lem:info-constraint}
\end{lem}
\begin{proof}
  Given $\pi \succeq 0$, it can be easily seen that $d(f, g) := \sqrt{\IC(f,g,\pi)} = \| \vec{f} - \vec{g} \|_{M_\pi}$ ,
  where we denote $\vec{h} = (h_a)_{a \in \cA}$ and $M_\pi = \diag(\fr{\pi_a}{2\sig^2}: a \in \cA)$; therefore $d(f, g)$ is a Mahalanobis distance, hence satisfying triangle inequality.
  Specifically,
  \[ d(f,g) \leq d(f, f^*) + d(g, f^*) \leq 2 \max(d(f, f^*), d(g, f^*)). \]
  Squaring both sides, we have $d(f, g)^2 \leq 4 \max(d(f, f^*)^2, d(g, f^*)^2)$, which implies that
  \[ \IC(f,g,\pi) \leq \max( \IC^*(f,4\pi ), \IC^*(g,4\pi) ). \qedhere \]
\end{proof}

\begin{lem}\label{lem:ic-finite}
  In Algorithm~\ref{alg:crop}, for every $t$, if we do not enter \exploit, then $\exists f\in\cF_t: \IC^*(f,4\pi_t) \ge1$;
  furthermore, if $\Cf_t$ happens, then $\exists f\in\rcF_t: \IC^*(f,4\pi_t) \ge1$.
\end{lem}
\begin{proof}
  There are three cases: $\Cf_t, \Fs_t, \Fb_t$.
  \begin{enumerate}
    \item If $\Cf_t$, then by the definition of $\phi(\barf_t)$ and $\Cf_t$, there exists $f\in\rcF_t$ such that $\IC(f,\barf_t, \phi(\barf_t))\ge1$.
    As $\barf_t$ is also in $\rcF_t$, by \Cref{lem:info-constraint}, $\exists f\in\rcF_t: \IC^*(f,4\phi(\barf_t)) \ge 1$.
    Therefore, we can state $\exists f\in\cF_t: \IC^*(f,4\phi(\barf_t)) \ge 1$.
    
    \item If $\Fs_t$, with a similar logic, either $\forall f\in\tilcF_t: \IC^*(f,4\gam(\barf_t))\ge1$ or $\IC^*(\barf_t, 4\gam(\barf_t)) \geq 1$. As $\tilcF_t \subseteq \cF_t$ and $\barf_t \in \cF_t$,
    we have $\exists f\in\cF_t: \IC^*(f,4\gam(\barf_t)) \ge 1$ as well.
    
    \item If $\Fb_t$, then with the same reasoning as the $\Fs_t$ case, either $\forall f\in\tilcF_t: \IC^*(f,4\psi(\barf_t))\ge1$ or $\IC^*(\barf_t, 4\psi(\barf_t)) \geq 1$, so we have $\exists f\in\cF_t: \IC^*(f,4\psi(\barf_t)) \ge 1$ as well. \qedhere
  \end{enumerate}
\end{proof}

Recall that we use $a(t) := a_t$ to avoid double subscripts.
\begin{lem}\label{lem:pi}
  If Algorithm~\ref{alg:crop} enters state $\Cf_t$, $\Fs_t$, or $\Fb_t$,  we have $\pi_t \neq 0$.
  In addition, $\pi_{a(t)} \neq 0$ holds with probability 1.
\end{lem}
\begin{proof}
  We consider three cases:
  \begin{itemize}
    \item If $\Cf_t$, there must exist $g \in \cE(\barf_t) \sm \{\barf_t\}, \gam(g) \not \propto \gam(f)$.
    By the definition of $\phi$, we must have $\IC(\barf_t, g, \phi(\barf_t)) \geq 1$, implying that $\pi_t = \phi(\barf_t) \neq 0$.
    
    \item If $\Fs_t$, then by the entering condition and the fact that $\tilf_t \in \tilcF_t$, $\IC(\barf_t, \tilf_t, \gamma(\barf_t)) \geq 1$ , implying that $\pi_t = \gamma(\barf_t) \neq 0$.
    
    \item If $\Fb_t$, then by the definition of $\psi$ and the fact that $\mu^*(\tilf_t) \geq \mu^*(\barf_t)$ and $a^*(\tilf_t) \neq a^*(\barf_t)$  , we must have $\forall \tilf \in\tilcF_t$, $\IC(\barf_t, \tilf, \psi(\barf_t)) \geq 1$, implying that $\pi_t = \psi(\barf_t) \neq 0$.
  \end{itemize}
  For the other claim, if $\pi_{a(t)} = 0$, then we have $\fr{T_{a(t)}(t-1)}{\pi_{t,a(t)}} = \infty$, so the selection of $a(t)$ implies that $\pi_{t} = 0$, which is impossible by the first claim here.
\end{proof}

We now present a useful lemma that formalizes the intuition that tracking (line~\ref{line:armpull}) controls the number of arm pulls and thus the statistical power (i.e., information) to distinguish $f^*$ from the rest.
We remark that the key variable below is $\zeta$, which appears three times in the LHS below.
\begin{lem}
  Let $\zeta\in\lbrack 0,\infty\rparen^K$. Then, for any hypothesis $f$,
  \begin{align*}
  \wbar{\Ex_t}, \pi_t \propto \zeta, a_t = a, T_a(t-1) \geq \rho \zeta_a,  \IC^*(f, \zeta) \geq c \implies \IC^*(f, T(t-1)) \geq \rho c~.
  \end{align*}
  \label{lem:ic-track}
\end{lem}
\begin{proof}
  By the definition of
  $a_t = \argmin_{a} \frac{T_a(t-1)}{\pi_{t,a}} = \argmin_{a} \frac{T_a(t-1)}{\zeta_a}$ and the condition that $T_a(t-1) \geq \rho \zeta_a$,
  we have, for every $b\in \cA$,
  \begin{align*}
  \fr{T_b(t-1)}{\zeta_b} \geq \fr{T_a(t-1)}{\zeta_a} \geq \rho~.
  \end{align*}
  This implies that
  \begin{equation}
  T(t-1) \succeq \rho \zeta.
  \label{eqn:t-succ}
  \end{equation}
  As a consequence, $\IC^*(f, T(t-1)) \ge \IC^*(f,\rho \zeta) = \rho \IC^*(f,\zeta)\ge \rho c$.
\end{proof}

As an application of Lemma~\ref{lem:ic-track}, we have the following lemma that will be useful for bounding the number of times different branches of \crop ($\Cf_t$,$\Fs_t$, and $\Fb_t$) are entered.

\begin{lem}
  The following statements hold:
  \begin{enumerate}[(i)]
    \item For any $a \in \cA$, $t \in \NN$ and $\rho > 0$,
    \[
    \PP( \Fs_t, \pi_t \propto \gamma_a(f^*), a_t = a, T_a(t-1) \geq \rho \gamma_a(f^*), \tilcF_t \subseteq \cC^* )
    \leq
    \abs{\cF} \cdot \exp\del{-\frac{\sig^2\rho - \beta_t}{4\sig^2}}.
    \]
    \label{item:gamma-star-elim}
    
    \item For any $a \in \cA$, $t \in \NN$ and $\rho> 0$,
    \[ \PP( \exists s \in [t] \centerdot \wbar{\Ex_s}, a_s = a, T_a(s-1) \geq \rho \pi_{s,a} ) \leq \abs{\cF} \cdot \exp \del{ - \frac{\sig^2\rho - 4\beta_t }{16\sig^2}} . \]
    \label{item:pi-elim}
    
    \item For any $a \in \cA$, $t \in \NN$ and $\rho> 0$,
    \[ \PP( \exists s \in [t] \centerdot \Cf_s, a_s = a, T_a(s-1) \geq \rho \pi_{s,a} ) \leq \abs{\cF} \cdot \exp \del{ -\frac{ \sig^2\rho - 4\rbeta_t }{16\sig^2} } . \]
    \label{item:phi-elim}
  \end{enumerate}
  \label{lem:tracking}
\end{lem}
\begin{proof} \phantom{a}
  \begin{enumerate}[(i)]
    \item If the event inside $\PP(\cdot)$ happens, we have the following: as $\tilcF_t \subseteq \cC^*$ , there must exists some $f_0 \in \cF_t$ such that $\IC(f_0, f^*,\gamma(f^*)) \geq 1$.
    As $f_0 \in\cF_t$, $L_{t-1}(f_0) - \min_{g\in\cF}L_{t-1}(g) \leq \beta_t \implies L_{t-1}(f_0) - L_{t-1}(f^*) \leq \beta_t$.
    Taking $\zeta = \gamma(f^*)$ in Lemma~\ref{lem:ic-track}, we have $\IC(f_0, f^*, T(t-1)) \geq \rho$.
    Therefore,
    \begin{align*}
    & \PP( \Fs_t, \pi_t \propto \gamma_a(f^*), a_t = a, T_a(t-1) \geq \rho \gamma_a(f^*), \tilcF_t \subset \cC^* ) 
    \\&\leq \PP\del{ \exists f \in \cF \centerdot L_{t-1}(f) - L_{t-1}(f^*) \leq \beta_t, \IC(f, f^*, T(t-1)) \geq \rho}
    \\&\leq\sum_{f \in \cF} \PP(L_{t-1}(f) - L_{t-1}(f^*) \leq \beta_t, \IC(f, f^*, T(t-1)) \geq \rho)
    \\&\leq \abs{\cF} \exp\del{-\frac{\sig^2\rho - \beta_t}{4\sig^2}}, 
    \end{align*}
    where the second inequality is from union bound, and the last inequality is from Lemma~\ref{lem:conc-main}\ref{item:ic-elim} and algebra.
    
    \item If the event inside $\PP(\cdot)$ happens, we have the following: there exists $s_0 \in [t]$, such that $\wbar{\Ex_{s_0}}$ happens, $a_{s_0} = a$, and $T_a(s_0-1) \geq \rho \pi_{s_0,a}$.
    From Lemma~\ref{lem:ic-finite}, there exists $f_0 \in \cF_{s_0} \subseteq \cbr{f: L_{s_0-1}(f) - L_{s_0-1}(f^*) \leq \beta_{s_0}}$ such that $\IC(f_0, f^*,4\pi_{s_0}) \geq 1$, implying that $\IC(f_0, f^*,\pi_{s_0}) \geq 1/4$.
    Taking $\zeta = \pi_{s_0}$ in Lemma~\ref{lem:ic-track}, we have that $\IC(f_0, f^*, T(s_0-1)) \geq \rho/4$. Therefore,
    \begin{align*}
    & \PP(  \exists s \in [t] \centerdot \barEx_s, a_s = a, T_a(s-1) \geq \rho \pi_{s,a} ) 
    \\&\le \PP( \exists s \in [t], f \in \cF \centerdot L_{s-1}(f) - L_{s-1}(f^*) \leq \beta_s, \IC(f, f^*, T(s-1)) \geq \rho/4 ) 
    \\&\le \PP( \exists s \in [t], f \in \cF \centerdot L_{s-1}(f) - L_{s-1}(f^*) \leq \beta_t, \IC(f, f^*, T(s-1)) \geq \rho/4 ) 
    \\&\le \sum_{f \in \cF} \PP(\exists s \in [t] \centerdot L_{s-1}(f) - L_{s-1}(f^*) \leq \beta_t, \IC(f, f^*, T(s-1)) \geq \rho/4) 
    \\&\le \abs{\cF} \exp\del{-\frac{\sig^2\rho - 4\beta_t}{16\sig^2}},
    \end{align*}
    where the second inequality uses the fact that $\beta_s \leq \beta_t$ for all $s \in [t]$; the third inequality is by union bound; the last inequality follows from Lemma~\ref{lem:conc-main}\ref{item:ic-elim} and algebra.
    
    \item If the event inside $\PP(\cdot)$ happens, we have the following: there exists $s_0 \in [t]$, such that $\Cf_{s_0}$ happens, $a_{s_0} = a$, and $T_a(s_0-1) \geq \rho \pi_{s_0,a}$.
    As $\Cf_{s_0}$  happens, by Lemma~\ref{lem:ic-finite}, there exists $f_0$ in $\rcF_{s_0} \subseteq \cbr{f\in\cF: L_{s_0-1}(f) - L_{s_0-1}(\barf_{s_0}) \leq \rbeta_{s_0}}$ such that $\IC(f_0, f^*, 4\pi_{s_0}) \geq 1$.
    Taking $\zeta = \pi_{s_0} = \phi(\barf_{s_0})$ in Lemma~\ref{lem:ic-track}, we have for that $f_0$, $\IC(f_0, f^*, T(s_0-1)) \geq \rho/4$. Therefore,
    \begin{align*}
    & \PP( \exists s \in [t] \centerdot \Cf_s, a_s = a, T_a(s-1) \geq \rho \pi_{s,a} ) 
    \\&\le \PP( \exists s \in [t], f \in \cF \centerdot L_{s-1}(f) - L_{s-1}(f^*) \leq \rbeta_s, \IC(f, f^*, T(s-1)) \geq \rho/4) 
    \\&\le \PP( \exists s \in [t], f \in \cF \centerdot L_{s-1}(f) - L_{s-1}(f^*) \leq \rbeta_t, \IC(f, f^*, T(s-1)) \geq \rho/4) 
    \\&\le \sum_{f \in \cF} \PP( \exists s \in [t]  \centerdot L_{s-1}(f) - L_{s-1}(f^*) \leq \rbeta_t, \IC(f, f^*, T(s-1)) \geq \rho/4) 
    \\&\le \abs{\cF} \exp\del{-\frac{\sig^2\rho - 4\rbeta_t}{16\sig^2}}.
    \end{align*}
    where the second inequality uses the fact that $\rbeta_s \leq \rbeta_t$ for all $s \in [t]$; the third inequality is by union bound; the last inequality follows from Lemma~\ref{lem:conc-main}\ref{item:ic-elim} and algebra. \qedhere
  \end{enumerate}
\end{proof}

\subsection{Main proofs}

Recall that $f^*$ is the ground truth mean rewards unless mentioned otherwise.
Throughout, we use shorthands for the ground truth: $\blue{a^* }:= a^*(f^*)$, $\blue{\mu^*} := \mu^*(f^*)$, and $\blue{\Delta_a} := \Delta_a(f^*)$.

Recall that we have define $\psi(\cG)$, $\phi(\cG)$ and $K_\psi$ in the main text.
Throughout we frequently use the notation $\blue{\cA_\zeta} = \cbr{a\in\cA: \zeta_a > 0}$ for vector $\zeta \in \lbrack0,\infty\rparen^K$.

Unlike observable states such as $\Ex_t$, $\Cf_t$, $\Fs_t$, and $\Fb_t$, there are hidden states that become useful for the purpose of analysis.
Based on the relationship between the hypothesis sets constructed by \crop and the hypothesis classes related to the ground truth hypothesis $f^*$, we define four {\em hidden states} of $\crop$:
\begin{enumerate}
  \item 'B'ad: : $\blue{B_t} = \one\{f^* \not\in \cF_t\}$.
  \item  $\text{Strongly steady state: } \blue{S_t^+} =\{\tilcF_t \subseteq \cC^*, \barf_t\in\cE^*, \gam(\barf_t) \propto\gam^*\} $.
  \item $\text{Weakly steady state: }  \blue{S_t^0} = \{\tilcF_t \subseteq \cC^*, \barf_t\in\cE^*, \gam(\barf_t)\not\propto\gam^*\}$.
  \item $\text{Non-steady state: } \blue{S^-_t}= \{ \tilcF_t \not\subseteq \cC^* ~\vee~ \barf_t \notin \cE^* \}~$.
\end{enumerate}
Note that the last three states forms a partition of the sample space, and can potentially overlap with $B_t$. In addition, if $\gamma^* = 0$, this would imply that $\cC^* = \emptyset$; in this case, states $S_t^+$ and $S_t^-$ will not ever be entered.

We first show a simple lemma that explains how non-steady states are related to having docile hypotheses in $\cF_t$.
\begin{lem}\label{lem:nonsteady}
  Suppose $\barB_t$ happens. Then, the event $S^-_t$ implies that $\exists f \in \cD^* \centerdot f \in \cF_t$.
\end{lem}
\begin{proof}
  Recall that $S^-_t =  \cbr{ \barf_t \in \cE^* \vee \tilcF_t \not\subseteq \cC^*} =  \cbr{ \barf_t \in \cD^*} \cup \cbr{ \barf_t \in \cC^*}  \cup \cbr{ \barf_t \in \cE^* , \tilcF_t \not\subseteq \cC^*}$.
  In addition, $\barB_t$ gives that $f^*$ is in $\cF_t$.
  These imply one of the following:
  \begin{enumerate}
    \item $\barf_t \in \cD^*$; in this case, we are done.
    \item $\barf_t \in \cC^*$. We first note that in this case, by the definition of $\cC^*$ along with the unique best arm assumption (defined in Section~\ref{sec:intro}), we have $\forall f\in\cC^*, \mu^*(f) > \mu^*(f^*) = \mu^*$; this implies that $\barmu_t = \mu^*(\barf_t) > \mu^*$.
    We consider two further subcases:
    
    \begin{enumerate}
      \item If $\tila_t \neq a^*$, then by the definition of $(\bara_t, \barmu_t) = \argmin_{(a,\mu) \in \cB_t: a \neq \tila_t} \mu$, the range of $(a,\mu)$'s in the minimum includes $(a^*,\mu^*)$, we have $\mu^*(\barf_t) = \barmu_t \le \mu^*$.
      This contradicts with our premise that $\barmu_t > \mu^*$.
      
      \item If $\tila_t = a^*$, then consider any hypothesis $f_0 \in \tilcF_t$. We have
      \[ f_0(a^*) = \tilmu_t \geq \barmu_t > \mu^* = f^*(a^*), \]
      implying that $f_0 \in \cD^*$.
      
    \end{enumerate}
    
    \item $\barf_t \in \cE^* \wedge \tilcF_t \not\subseteq \cC^*$.
    In this case, there must exist an element $f_0 \in \tilcF_t$ such that $f_0 \in \cD^*$ or $f_0 \in \cE^*$. We claim that the latter cannot happen.
    To see why, by the definition of $\barcF_t$, it must be true that $\barcF_t$ and $\tilcF_t$ belong to two different equivalence classes induced by relationship $\sim$. In addition, as $\barf_t \in \cE^*$, $\tilcF_t$ is a subset of the equivalence class $\cE^*$.
    This implies that $\tilcF_t \cap \cE^* = \emptyset$. Therefore, $f_0 \in \cD^*$ must hold.
  \end{enumerate}
  In summary, in all cases, we have that there exists some $f$ in $\cF_t$ such that $f \in \cD^*$.
\end{proof}

We will bound the expected regret of \crop by a case-by-case analysis on the combination of observable and hidden states at each time step:
\begin{align}\label{eq:planofproof}
\begin{aligned}
\EE \Reg_n
  &= \EE\sum_a \Delta_a \sum_{t=1}^n \one\{a_t = a\} 
  \\&\leq
    \EE\sum_a \Delta_a \sum_{t=1}^n \one\{a_t = a, B_t \}
    +
    \EE\sum_a \Delta_a \sum_{t=1}^n \one\{a_t = a, \barB_t, \Ex_t \}
  \\&\quad +
  \underbrace{\EE\sum_a \Delta_a \sum_{t=1}^n \one\{a_t = a, \barB_t, \overline{\Ex_t}, S_t^- \}}_{\zo}
  +
  \underbrace{\EE\sum_a \Delta_a \sum_{t=1}^n \one\{a_t = a, \barB_t, \overline{\Ex_t}, \wbar{S_t^-} \}}_{\zt}
\end{aligned}
\end{align}

Note that the first two terms are easy to bound.
First,
\begin{align*}
\EE\sum_a \Delta_a \sum_{t=1}^n \one\{a_t = a\}
&\leq \Delta_{\max} \EE \sum_{t=1}^n \one\{B_t\}
\\&\leq \Delta_{\max}\sum_{t=1}^n \PP\left(L_{t-1}(f^*) - \min_{f\in\cF}L_{t-1}(f) > \beta_t\right)
\\&\le  \Delta_{\max}|\cF| \sum_{t=1}^n \fr{1}{zt^\alpha}
~\le 2\Delta_{\max} ~.
\end{align*}
where the last inequality is from \Cref{lem:conc-main}\ref{item:fstar-in} and by the values of $z$ and $\alpha$ stated in Theorem~\ref{thm:main-appendix}.

Second, if $\barB_t$ and $\Ex_t$ happens, $a_t = a^*(\cF_t) = a^*$. Therefore,
\[ \EE\sum_a \Delta_a \sum_{t=1}^n \one\{a_t = a, \barB_t, \Ex_t \} = \EE \Delta_{a^*} \sum_{t=1}^n \one\{a_t = a^*, \barB_t, \Ex_t \} = 0. \]

To bound the third term $\zo$, we use Lemma~\ref{lem:transient} in Section~\ref{subsec:each}.

For $\zt$, we decompose $\zt$ to a few more sub-terms. We first have the following claim:
\begin{claim}
  If $S_t^-$ does not happen, then $\Fb_t$  does not happen, i.e. either $\Cf_t$ or $\Fs_t$ happens.
\end{claim}
\begin{proof}
  If $S_t^-$ does not happen, then we have $\barf_t \in \cE^*$ and $\tilcF_t \subseteq\cC^*$ both hold.
  These imply that $\cC^* = \cC(\barf_t)$, and consequently, $\tilcF_t \subseteq\cC(\barf_t)$. But this would imply that the condition of line~\ref{line:fs-cond} is satisfied and the state $\Fb_t$ will not be entered.
\end{proof}

The above claim indicates that $\zt$ can be bounded by:
\begin{align*}
\zt 
& = 
\EE\sum_a \Delta_a \sum_{t=1}^n \one\{a_t = a, \Fs_t, \wbar{S_t^-} \}
+
\EE\sum_a \Delta_a \sum_{t=1}^n \one\{a_t = a, \Cf_t, \wbar{S_t^-} \}
\\
& = \underbrace{\EE\sum_a \Delta_a \sum_{t=1}^n \one\{a_t = a, \Fs_t, S_t^+ \}}_{\zta}
+ \underbrace{\EE\sum_a \Delta_a \sum_{t=1}^n \one\{a_t = a, \Cf_t, \wbar{S^-_t}\} }_{\ztb}
\\&+ \underbrace{\EE\sum_a \Delta_a \sum_{t=1}^n \one\{a_t = a, \Fs_t,  S_t^0 \} }_{\ztc}.
\end{align*}
Now, using Lemmas~\ref{lem:z2-a},~\ref{lem:z2-b}, and~\ref{lem:z2-c} in Section~\ref{subsec:each}, we get
\begin{align*}
\zt
&\leq
16 P_1 \cdot  (\ln(n) + \ln\abs{\cF}) + 2 K_\psi \Delta_{\max}
+
64  P_2 \cdot \del{ \ln(\ln(n) ) + \ln(|\cF|) }  + 2 K_\psi \Delta_{\max} 
\\&\qquad+
160   \cd \del{\sum_a \Delta_a \gam_a(\cE^*)} \ln \abs{\cF} + 5 K_\psi \Delta_{\max} 
\\&\le
O\del{  P_1 \ln n
  +  P_2 \ln(\ln n)
  +  P_3 \cdot \ln |\cF|
  + K_\psi \Delta_{\max} }
\end{align*}
where the second inequality is by algebra, and the fact that $P_3 \ge P_1 \vee P_2 \vee  (\sum_a \Delta_a \gam_a(\cE^*))$, which in turn is from the constraint in~\eqref{eq:psi} we have $\psi_a(f) \geq \phi_a(f) \vee \gam_a(f)$ for all $a \in \cA$ and $f \in \cF$.

Combining the above bound on $\zt$ with Lemma~\ref{lem:transient} and Equation~\eqref{eq:planofproof}, we can bound the regret of \crop as follows:
\begin{align*}
\Reg_n
&\le \zo + \zt + 2\Delta_{\max}
\\&\le  O\del{   P_3 (\ln\abs{\cF} + \ln(Q_1) ) + K_\psi \Delta_{\max} }
+ O\del{  P_1 \ln n
  +   P_2 \ln(\ln n)
  +   P_3 \cdot \ln |\cF|
  + K_\psi \Delta_{\max} }
\\&\qquad+  2\Delta_{\max} 
\\&= O \del{ P_1 \ln(n) + P_2 \ln(\ln(n))
  + P_3 \del{\ln(|\cF|)  +  \ln\del{ Q_1 }}
  + K_\psi \Delta_{\max} }~.
\end{align*}
If $\gamma^* = 0$, we have $P_1 = \sum_{a \in \cA} \Delta_a \gamma_a^* = 0$.
In addition, we must have $\cC^* = \emptyset$, implying that for all $f \in \cE^*$, $\gamma^*(f) = 0$.
This in turn implies that for all $f, g$ in $\cE^*$, $f \propto g$ is trivially true, and consequently $\phi(f) = 0$ for
all $f \in \cE^*$. Therefore, $\phi(\cE^*) = 0$ and $P_2 = \sum_{a \in \cA} \Delta_a \phi_a(\cE^*) = 0$.
The proof of Theorem~\ref{thm:main} is complete. \qed

\subsection{Bounding the regret in each individual case}
\label{subsec:each}

\subsubsection{Bounding $\zta$.}
\label{subsec:zta}

Recall that we use the shortcut $\blue{\gam^*} := \gam(f^*)$. In addition, we have defined $P_1 = \sum_a \Delta_a \gamma_a^*$ and $\cA_{\gam^*} = \cbr{a\in\cA: \gam^*_a \neq 0}$. Note $|\cA_{\gam^*}| \le K_\psi$.
\begin{lem}
  \[ \EE\sum_a \Delta_a \sum_{t=1}^n \one\{a_t = a, \Fs_t, S_t^+ \} \leq 16 P_1 \cdot  (\ln(n) + \ln\abs{\cF}) + 2 |\cA_{\gam^*}| \Delta_{\max}. \]
  \label{lem:z2-a}
\end{lem}

\begin{proof}
  By linearity of expectation,
  \begin{equation}
  \EE\sum_a \Delta_a \sum_{t=1}^n \one\{a_t = a, \Fs_t, S_t^+ \} = \sum_a \Delta_a \EE  \sum_{t=1}^n \one\{a_t = a, \Fs_t, S_t^+ \}.
  \label{eqn:b1-decomp}
  \end{equation}
  
By Lemma~\ref{lem:pi}, $\cbr{a_t = a, \Fs_t, S_t^+}$ will happen only for those $a$'s in $\cA_{\gam^*}$; thus only the arms in $\cA_{\gam^*}$ will contribute to the sum, which we focus on, hereafter.
  
  With foresight, we pick $\blue{q_{1,a}} = \lt\lceil  2 \fr{\beta_n}{\sig^2} \cdot \gamma_a(f^*)  \rt\rceil$.
  Then,
  \begin{align*}
  &\EE \sum_{t=1}^n \one\{a_t=a, \Fs_t, S^+_t\}
  \\&\le  q_{1,a} +  \sum_{t=q_{1,a}+1}^n \PP\del{a_t=a, \Fs_t, S^+_t, T_a(t-1) \ge q_{1,a}}
  \\&\le  q_{1,a} +  \sum_{t=q_{1,a}+1}^n \PP\del{ \Fs_t, \pi_t \propto \gamma_a(f^*), a_t = a, T_a(t-1) \geq 2 \fr{\beta_n}{\sig^2} \gamma_a(f^*), \tilcF_t \subset \cC^* }
  \\&\le q_{1,a} + \sum_{t = q_{1,a}+1}^n \abs{\cF} \cdot \exp\del{- \frac{2 \beta_n - \beta_t}{4\sig^2} }
  \\&\le  2 \cdot \fr{\beta_n}{\sig^2} \gamma_a(f^*) + 1 + \sum_{t=1}^n \abs{\cF} \frac{1}{n^2 \abs{\cF}}
  \\&\le 2  \cdot \fr{\beta_n}{\sig^2}\gamma_a(f^*) + 2.
  \end{align*}
  where the first inequality uses \Cref{lem:boundarmpullzero} with $\tau = q_{1,a}$; the second inequality uses the fact that if $\Fs_t$ and $S_t^+$  happens, $\pi_t = \gamma(\barf_t) \propto \gamma(f^*)$, and $q_{1,a} \geq  2 \fr{\beta_n}{\sig^2} \cdot \gamma_a(f^*)$; the third inequality uses Lemma~\ref{lem:tracking}\ref{item:gamma-star-elim}; the last two inequalities are by the fact that $\beta_t\le \beta_n$ and algebra.
  
  Continuing Equation~\eqref{eqn:b1-decomp}, summing over all actions $a \in \cA_{\gam^*}$ with weight $\Delta_a$'s, we have
  \begin{align*}
   \sum_a \Delta_a \EE  \sum_{t=1}^n \one\{a_t = a, \Fs_t \wedge S_t^+ \}
    & \leq 
  2 \sum_a \Delta_a \gamma_a^* \cdot \fr{\beta_n}{\sig^2} + 2 |\cA_{\gam^*}| \Delta_{\max} 
  \\&\leq 
  16 \del{ \sum_a \Delta_a \gamma_a^* } \cdot  (\ln(n) + \ln\abs{\cF}) + 2 |\cA_{\gam^*}| \Delta_{\max}.
  \end{align*}
  where the first inequality is from the fact that $\Delta_a \leq \Delta_{\max}$, and the second inequality is
  from the definition of $\beta_n$. 
  The lemma follows from the definition of $P_1$.
\end{proof}

\subsubsection{Bounding $\ztb$.}
\label{subsec:ztb}
Recall that $P_2 = \sum_a \Delta_a \phi_a(\cE^*)$ and $\cA_{\phi(\cE^*)} = \{a\in\cA: \phi_a(\cE^*) > 0\}$.
Note $|\cA_{\phi(\cE^*)} |\le K_\psi$.
\begin{lem}
  \[ \EE\sum_a \Delta_a \sum_{t=1}^n \one\cbr{a_t = a, \Cf_t, \wbar{S^-_t} } \leq 64 P_2 \cdot \del{ \ln(\ln(n) ) + \ln(|\cF|) }  + 2 |\cA_{\phi(\cE^*)}| \Delta_{\max}.  \]
  \label{lem:z2-b}
\end{lem}
\begin{proof}
  First, we note that by Lemma~\ref{lem:pi}, if $a$ is not in $\cA_{\phi(\cE^*)}$, it does not contribute to the sum, as $\Cf_t$ implies that only actions in $\cA_{\phi(\cE^*)}$ are taken with nonzero probability.
  
  Next, by the linearity of expectation, we rewrite the expectation as follows:
  \[
  \EE \sum_a \Delta_a \sum_{t=1}^n \one\{a_t = a, \Cf_t, \wbar{S^-_t} \}
  = \sum_{a \in \cA_{\phi(\cE^*)}} \Delta_a \EE \sbr{ \sum_{t=1}^n \one\{a_t = a, \Cf_t, \wbar{S^-_t} \} }
  \]
  
  For any $a \in \cA_{\phi(\cE^*)}$,
  \begin{align*}
  \EE \sum_{t=1}^n \one\{a_t = a, \Cf_t, \wbar{S^-_t} \}
  &\leq 1 + \EE \sum_{m=1}^\infty \one\{\exists t\le n \centerdot a_t = a, T_a(t-1) \geq m, \Cf_t, \wbar{S^-_t} \}
  \\&\leq 1 + \sum_{m=1}^\infty \PP \del{  \exists t \le n\centerdot a_t = a, \Cf_t, T_a(t-1) \geq \frac{m}{\phi_a(\cE^*)} \pi_{t,a} }
  \\&\leq 1 + \sum_{m=1}^\infty \min\del{ 1, \exp\del{-\frac{ \frac{\sig^2 m}{\phi_a(\cE^*)} - 4 \rbeta_n }{16\sig^2}} }
  \end{align*}
  where the first inequality is from Lemma~\ref{lem:boundarmpull} with $\tau = 1$; the second inequality is from the fact that if $T_a(t-1) \geq m$ and $\wbar{S^-_t}$ happens, then $\barf_t \in \cE^*$, and therefore $T_a(t-1) \geq \frac{m}{\phi_a(\cE^*)} \phi_a(\barf_t) = \frac{m}{\phi_a(\cE^*)} \pi_{t,a}$; the third inequality is from Lemma~\ref{lem:tracking}\ref{item:phi-elim} and $\PP(A) \leq 1$ for any event $A$.
  
  We remark that naively applying Lemma~\ref{lem:boundarmpullzero} instead of Lemma~\ref{lem:boundarmpull} as used in the case $\zta$ (also used in the proofs of UCB~\cite{auer02finite} and UCB-S~\cite{lattimore14bounded}), does not lead to the desired bound because of the aggressive confidence level of $\rcF_t$.
  
  Denote by $\blue{N_m} = \min\del{ 1, \exp\del{-\frac{ \frac{\sig^2m}{\phi_a(\cE^*)} - 4\rbeta_n }{16\sig^2}} }$ and let $\blue{m_0} =\lceil 4 \phi_a(\cE^*) \fr{\rbeta_n}{\sig^2} \rceil$.
  
  For $m \leq m_0 - 1$, we use the fact that $N_m \leq 1$. For $m \geq m_0$, $\cbr{N_m}_{m \geq m_0}$ is a geometric progression with initial value $N_{m_0} \leq 1$ and common ratio $\exp(-\frac{1}{16 \phi_a(\cE^*)})$. This implies that
  \begin{align*}
  1 + \sum_{m=1}^\infty N_m
  ~\leq  m_0 + \sum_{m=m_0}^\infty N_m
  &\leq m_0 + \frac{1}{1 - \exp(-\frac{1}{16 \phi_a(\cE^*)})} \\
  &\leq  1 + 4\phi_a(\cE^*) \fr{\rbeta_n}{\sig^2} + (1 + 16 \phi_a(\cE^*)) \\
  &\leq  2 + 64\phi_a(\cE^*)\ln(|\cF| \ln(n) ).
  \end{align*}
  where the first two inequalities are by algebra, the third inequality is from the definition of $m_0$ and the elementary fact that $\frac1{1 - \exp\del{-1/x}} = 1+\fr{1}{\exp(1/x) - 1} \leq  1+\fr{1}{((1/x)+1) - 1} = 1+x$ for $x > 0$; the last inequality is from the definition of $\rbeta_n$, ~$\abs{\cF} \geq 2$, ~$n \ge 2$  and algebra.
  Consequently,
  \begin{align*}
  & \sum_{a \in \cA_{\phi(\cE^*)}} \Delta_a \EE \sbr{ \sum_{t=1}^n \one\cbr{a_t = a, \Cf_t, \wbar{S^-_t} } }
  \\&\le  \sum_{a \in \cA_{\phi(\cE^*)}} \Delta_a \del{2 + 64 \phi_a(\cE^*) \ln(|\cF| \ln(n) )}
  \\&\le 64 \del{ \sum_a \Delta_a \phi_a(\cE^*) } \cdot \del{ \ln(\ln(n) ) + \ln(|\cF|) }  + 2 |\cA_{\phi(\cE^*)}| \Delta_{\max},
  \end{align*}
  where the second inequality uses the facts that $\Delta_a \leq \Delta_{\max}$ and algebra. The lemma follows from the definition of $P_2$.
\end{proof}

\subsubsection{Bounding $\ztc$.}
\label{subsec:ztc}

Recall that $\cA_{\gamma(\cE^*)} = \{a\in\cA: \gam_a(\cE^*) >0 \}$. Note $|\cA_{\gamma(\cE^*)}| \le K_\psi$.
\begin{lem}
  \[ \EE\sum_a \Delta_a \sum_{t=1}^n \one\{a_t = a, \Fs_t, S_t^0 \} \leq 160 \del{\sum_a \Delta_a \gam_a(\cE^*)} \ln \abs{\cF} + 5 |\cA_{\gamma(\cE^*)}| \Delta_{\max}. \]
  \label{lem:z2-c}
\end{lem}
\begin{proof}
  First, we note that if $a$ is not in $\cA_{\gamma(\cE^*)}$, it does not contribute to the sum, as $\Fs_t$ implies that only actions in $\cA_{\gamma(\cE^*)}$ are taken with nonzero probability.
  
  By linearity of expectation,
  \begin{equation}
  \EE\sum_a \Delta_a \sum_{t=1}^n \one\{a_t = a, \Fs_t, S_t^0 \}
  \leq
  \sum_a \Delta_a \EE \sum_{t=1}^n \one\{a_t = a, \Fs_t, S_t^0 \}.
  \label{eqn:z3-decomp}
  \end{equation}
  
  For every $a \in \cA_{\gamma(\cE^*)}$, we will upper bound $\EE \sum_{t=1}^n \one\{a_t = a, \Fs_t, S_t^0 \}$. To this end, we will upper bound
  $\blue{C_{a, n_0}} := \EE \sum_{t=n_0+1}^{2 n_0} \one\{a_t = a, \Fs_t, S_t^0 \}$,
  for every $n_0 \in \{2^k: k \in \cbr{1,2,\ldots}\}$.
  
  We first note that if $\Fs_t$ and $S_t^0$ both happen, then by the definition of $S_t^0$, $\barf_t \sim f^*$, and $\rcF_t \subseteq \barcF_t \subseteq \cE^*$; we also have $\gamma(\barf_t) \not \propto \gamma(f^*)$.
  In addition, by the definition of $\Fs_t$, for all $f, g \in \rcF_t$, $\gamma(f) \propto \gamma(g)$.  Therefore, it must be the case that $f^* \notin \rcF_t$, implying that $\exists f \in \cE^* \centerdot L_{t-1}(f^*) - L_{t-1}(f) > \rbeta_t$.
  We use this observation in the subsequent proof that we call ``regret peeling''.
  
  With foresight, we pick $\blue{u_a} = \lceil 2 \gamma_a(\cE^*) \fr{ \beta_{2 n_0} }{\sig^2} \rceil$.
  We can bound $C_{a,n_0}$ as follows:
  \begin{align*}
  C_{a, n_0} & = \EE \sum_{t=n_0+1}^{2 n_0} \one\{a_t = a, \Fs_t, S_t^0 \} \\
  & \leq \EE \sum_{t=n_0+1}^{2 n_0} \one\{a_t = a, \Fs_t, \exists f \in \cE^* \centerdot L_{t-1}(f^*) - L_{t-1}(f) > \rbeta_t \} \\
  & \leq \EE \sum_{t=n_0+1}^{2 n_0} \one\{a_t = a, \Fs_t, \exists f \in \cE^* \centerdot L_{t-1}(f^*) - L_{t-1}(f) > \rbeta_{n_0} \} \\
  & \leq \EE \one\cbr{ \exists s \in \NN, f \in \cE^* \centerdot L_{s}(f^*) - L_{s}(f) > \rbeta_{n_0}} \cdot \sum_{t=n_0+1}^{2 n_0} \one\cbr{a_t = a, \Fs_t} \\
  & \leq \EE \one\cbr{ \exists s \in \NN, f \in \cE^* \centerdot L_{s}(f^*) - L_{s}(f) > \rbeta_{n_0}} \cdot \del{ u_a + \sum_{t=n_0+1}^{2 n_0} \one\cbr{a_t = a, \Fs_t, T_a(t-1) \geq u_a} } \\
  &\leq \PP \del{ \exists s \in \NN, f \in \cE^* \centerdot L_{s}(f^*) - L_{s}(f) > \rbeta_{n_0}} \cdot u_a + \sum_{t=n_0+1}^{2 n_0} \PP\del{a_t = a, \Fs_t, T_a(t-1) \geq u_a}
  \end{align*}
  where the second inequality uses the basic fact that $\rbeta_t > \rbeta_{n_0}$ for all $t \geq n_0+1$; the third inequality is from the basic fact that $\one\cbr{A, B} = \one\cbr{A} \cdot \one\cbr{B}$ and the fact that for predicate $p$, $p(t)$ implies $\exists s \centerdot p(s)$.
  
  The first term can be bounded by \Cref{lem:conc-main}\ref{item:fstar-in} and the union bound as follows:
  \begin{align*}
  & \PP \del{ \exists s \in \NN, f \in \cE^* \centerdot L_{s}(f^*) - L_{s}(f) > \rbeta_{n_0}} \cdot u_a
  \\&\leq  \abs{\cF} \exp\del{ -\frac{\rbeta_{n_0}}{4\sig^2} } \cdot u_a  
  \\&\leq  \frac{1}{(\log_2 n_0)^3} + 16 \gamma_a(\cE^*) \cdot \del{ \frac{\ln \abs{\cF}}{(\log_2 n_0)^3} + \frac{2}{(\log_2 n_0)^2} }.
  \end{align*}
  where the last inequality uses $n_0\ge2$ and $\fr{\ln(2 n_0)}{(\log_2(n_0))^3} \le 2\ln(2) \le 2$. 
  \begin{remark}
    We remark that the inequality above is the one that reflects our intuition that, even if we track a wrong $\gam(f)$ with $f\in \cE^*$ and suffer regret like $\sum_a \Delta_a \gam_a(f) \ln(t)$ up to time $t$ (that can be much larger than $\sum_a \Delta_a \gam^*_a \ln(t)$), such an event happens with small enough probability like $O(\fr{1}{\ln(t)})$.
    Therefore, \textit{in expectation}, this event contributes to the regret only as a finite term w.r.t. $n$.
    This intuition is manifested in the proof in a bit more complicated way, unfortunately, because the algorithm is designed to enjoy an anytime regret bound rather than the fixed-budget setting. 
    Specifically, the failure rate of the confidence set $\rcF_t$ changes over time, and we use the common technique called ``peeling device'' from concentration of measure to deal with it. 
    If we knew the time horizon $n$, then one can set $\rbeta_t = 4\sigma^2 \ln(\abs{\cF} \ln n)$ for all $t$ to obtain the same guarantee.
  \end{remark}
  
  Meanwhile, each subterm in the second term can be bounded using Lemma~\ref{lem:tracking}\ref{item:pi-elim} as follows:
  \begin{align*}
  \PP \del{ a_t = a, \Fs_t, T_a(t-1) \geq u_a }
    &\leq  \PP \del{ a_t =a, \Fs_t, T_a(t-1) \geq \frac{u_a}{\gamma_a(\cE^*)} \pi_{t,a} } 
  \\&\leq \abs{\cF} \cdot \exp\del{-\frac{\sig^2\frac{u_a}{\gamma_a(\cE^*)} - \beta_t }{4\sig^2}}
  \leq \frac{1}{(2n_0)^2}.
  \end{align*}
  where the last inequality is from the definition of $u_a$ and $\beta_{2n_0} \geq \beta_t$.  This implies that
  \[ \sum_{t=n_0+1}^{2 n_0} \PP\del{a_t = a, \Fs_t, T_a(t-1) \geq u_a} \leq \frac{1}{4n_0}. \]
  
  In summary, we have
  \[ C_{a, n_0} \leq \del{ \frac{1}{(\log_2 n_0)^3} + \frac{1}{4n_0} } + 16 \sig^2 \gamma_a(\cE^*) \cdot \del{ \frac{\ln \abs{\cF}}{(\log_2 n_0)^3} + \frac{2}{(\log_2 n_0)^2} }. \]
  
  Now, we can upper bound $\EE \sum_{t=1}^n \one\{a_t = a, \Fs_t, S_t^0 \}$ as follows:
  \begin{align*}
  \EE \sum_{t=1}^n \one\{a_t = a, \Fs_t \wedge S_t^0 \}
  & \leq 2 + \sum_{k=1}^\infty C_{a, 2^k} \\
  & \leq 2 + \sum_{k=1}^\infty \del{ \frac{1}{k^3} + \frac{1}{2^k} }
  +  16 \gamma_a(\cE^*) \cdot \del{ \sum_{k=1}^\infty  \frac{\ln \abs{\cF}}{k^3} + \frac{2}{k^2} } \\
  & \leq 5 + 160  \gamma_a(\cE^*) \cdot \ln \abs{\cF},
  \end{align*}
  where the last inequality is by algebra and $1 \leq 2\ln \abs{\cF}$ due to $|\cF|\ge2$.
  
  Using the bound above, continuing Equation~\eqref{eqn:z3-decomp}, we have
  \begin{align*}
  \sum_{a \in \cA_{\gamma(\cE^*)}} \Delta_a \EE \sum_{t=1}^n \one\{a_t = a, \Fs_t \wedge S_t^0 \}
  &\leq \sum_{a \in \cA_{\gamma(\cE^*)}} \Delta_a (5 + 160 \gamma_a(\cE^*) \cdot \ln \abs{\cF})
  \\&\leq 160 \del{\sum_a \Delta_a \gamma_a(\cE^*) } \ln \abs{\cF} + 5 |\cA_{\gamma(\cE^*)}| \Delta_{\max},
  \end{align*}
  where the last inequality uses $\Delta_a \leq \Delta_{\max}$ for all $a \in \cA$.
\end{proof}

\subsubsection{Bounding $\zo$}
\label{subsec:zo}

Recall that $P_3 = \sum_a \Delta_a \psi_a(\cF) $, $\Lam_{\min} = \min_{f \in\cD^*} \fr{| f(a^*) - \mu^* |}{\sig}$ is the smallest information gap, and $Q_1 = \Lam_{\min}^{-2} + K_\psi(1+ \max_a \psi_a({\cF}))$.
\begin{lem}
  \begin{equation}
  \EE\sum_a \Delta_a \sum_{t=1}^n \one\{a_t = a, \barB_t, \overline{\Ex_t}, S_t^- \} \leq O\del{ P_3 (\ln\abs{\cF} + \ln(Q_1) ) + K_\psi \Delta_{\max} }.
  \label{eqn:transient}
  \end{equation}
  \label{lem:transient}
\end{lem}

\begin{proof}
  Recall that $\alpha = 2$. With foresight, define
  \begin{equation}
  \blue{\tau} = \max \cbr{t \in \NN_+:
    \del{\fr{t}{2} < 1+\fr{4}{\Lam_{\min}^2\sig^2} \beta_t}  ~\vee~
    \del{\fr{t}{4K_\psi} < 1+ 8  \fr{\beta_t}{\sig^2} (\max_a \psi_a(\cF))} 
  }.
  \label{eqn:tau}
  \end{equation}
  
  We upper bound the LHS of Equation~\eqref{eqn:transient} with three terms:
  \begin{align*}
  & \sum_a \Delta_a \EE \sum_{t=1}^n \one\{a_t = a, \barB_t, \overline{\Ex_t}, S_t^- \} 
  \\&\leq 
  \EE \sum_a \Delta_a \sum_{t=1}^\tau \one\{a_t = a, \barB_t, \overline{\Ex_t},  S_t^- \}
  +
  \EE \sum_{t=\tau+1}^n \sum_a \Delta_a  \one\{a_t = a, \barB_t,\overline{\Ex_t},  S_t^-, T_{a^*}(t-1) \geq \frac t 2 - 1 \} 
  \\& \quad +
  \EE \sum_{t=\tau+1}^n \sum_a \Delta_a  \one\{a_t = a, \barB_t, \overline{\Ex_t}, S_t^-, T_{a^*}(t-1) < \frac t 2 - 1 \}
  \\&\leq
  \underbrace{\sum_a \Delta_a \EE \sum_{t=1}^\tau \one\{a_t = a, \barB_t, \overline{\Ex_t}, S_t^- \}}_{\zoa}
  +
  \underbrace{\Delta_{\max} \sum_{t=\tau+1}^n \PP \del{ \barB_t, \overline{\Ex_t}, S_t^-, T_{a^*}(t-1) \geq \frac t 2 - 1 } }_{\zob} 
  \\&\quad +
  \underbrace{\Delta_{\max} \sum_{t=\tau+1}^n \PP \del{T_{a^*}(t-1) < \frac t 2 - 1 }}_{\zoc}
  \end{align*}
  where the first inequality is by algebra; the second inequality uses the fact that $\Delta_a \leq \Delta_{\max}$ for all $a$,
  and linearity of expectation.
  
  We bound each term respectively.

  \paragraph{Bounding $\zoa$.}
  Recall $\cA_\psi = \cbr{a\in\cA: \psi(\cF)>0}$.
  Define $\cA'_\psi = \cA_\psi \sm \cbr{a^*}$.
  First, we note that if $a$ is not in $\cA'_\psi$, by Lemma~\ref{lem:pi}, it does not contribute to the sum.
  This is because, if $\barB_t$ happens, the only arms being pulled is either $a^*$ or from  $\cA'_\psi$ (and $\Delta_{a^*} = 0$).
  
  With the choice of $\blue{q_{3,a}} = \lceil 8 \fr{\beta_\tau}{\sig^2} \psi_a(\cF) \rceil$, we have
  \begin{align*}
  &\EE \sum_{t=1}^\tau \one\{a_t = a, \barB_t, \overline{\Ex_t}, S_t^- \}
  \\&\leq \EE \sum_{t=1}^\tau \one\{a_t = a, \wbar{\Ex_t}, S_t^- \}
  \\&\leq q_{3,a} + \sum_{t=1}^\tau \PP \del{a_t = a, \wbar{\Ex_t}, T_a(t-1) \geq q_{3,a} }
  \tag*{($\because$~ \Cref{lem:boundarmpullzero})}
  \\&\leq q_{3,a} + \sum_{t=1}^\tau \PP \del{a_t = a, \wbar{\Ex_t}, T_a(t-1) \geq 8 \fr{\beta_\tau}{\sig^2} \pi_{t,a} }
  \tag*{($\because~  q_{3,a} \geq 8 \fr{\beta_\tau}{\sig^2} \psi_a(\cF)$)}
  \\&\leq q_{3,a} + \sum_{t=1}^\tau \abs{\cF} \exp\del{ - \frac{8\beta_\tau - 4\beta_t  }{16\sig^2} }
  \tag*{($\because$~ Lemma~\ref{lem:tracking}\ref{item:pi-elim})}
  \\&\leq 8 \beta_\tau \psi_a(\cF) + 1 + \sum_{t=1}^\tau \frac{1}{\tau^2}
  \tag*{($\because$~ $\beta_t\le\beta_\tau$,  the definition of $\beta_\tau$)}
  \\&\leq 16 \psi_a(\cF) (\ln\abs{\cF} + \ln \tau ) + 3
  \\&= O( \psi_a(\cF)\cd (\ln\abs{\cF} + \ln Q_1 ) + 1 ) ~,
  \end{align*}
  where the last inequality is from \Cref{lem:tau} below where we show $\ln \tau = O(\ln(Q_1) + \ln\ln\abs{\cF})$.
  
  Summing over all $a \in \cA'_\psi$, we have
  \[
  \sum_a \Delta_a \EE \sum_{t=1}^\tau \one\{a_t = a, \barB_t, \overline{\Ex_t}, S_t^- \}
  = O\del{ \del{ \sum_a \Delta_a \psi_a(\cF) } (\ln\abs{\cF} + \ln(Q_1) ) + K_\psi \Delta_{\max} }.
  \]
  
  \paragraph{Bounding $\zob$.}
  In subsequent derivations, we denote by $\blue{X_t} = \frac{4}{\Lambda_{\min}^2} \fr{\beta_t}{\sig^2}$.
  \begin{align*}
  & \sum_{t=\tau+1}^n \PP \del{ \barB_t, \overline{\Ex_t}, S_t^-, T_{a^*}(t-1) \geq \frac t 2 - 1 } \\
  & \leq \sum_{t=\tau+1}^n \PP \del{\barB_t, \overline{\Ex_t}, S_t^-, T_{a^*}(t-1) \geq X_t } \\
  & \leq \sum_{t=\tau+1}^n \PP \del{T_{a^*}(t-1) \geq X_t, \exists f \in \cD^* \centerdot L_{t-1}(f) - L_{t-1}(f^*) \leq \beta_t } \\
  & \leq \sum_{t=\tau+1}^n \PP \del{\exists f \in \cD^* \centerdot \IC^*(f, T(t-1)) \geq 2\fr{\beta_t}{\sig^2},  L_{t-1}(f) - L_{t-1}(f^*) \leq \beta_t } \\
  & \leq \sum_{t=\tau+1}^n \abs{\cF} \exp\del{-\frac{\beta_t}{4\sig^2}}  \\
  &\leq \sum_{t=1}^n \frac{1}{t^2} 
   ~\leq 2,
  \end{align*}

  where the first inequality is from the definition of $\tau$: for every $t > \tau$, $\frac t 2 - 1 \geq X_t$; the second inequality is
  from Lemma~\ref{lem:nonsteady}; the third inequality is from the observation that $\IC^*(f, T(t-1)) \geq \frac12 T_{a^*}(t-1) \Lambda_{a^*}(f)^2 \geq \frac12 T_{a^*}(t-1) \Lambda_{\min}^2 \geq 2\fr{\beta_t}{\sig^2}$; the fourth inequality is from Lemma~\ref{lem:conc-main}\ref{item:ic-elim}; the last two inequalities are by algebra.
  
  \paragraph{Bounding $\zoc$.} 
  We first bound $\PP \del{ T_{a^*}(t-1) < \frac t 2 - 1 }$ for each $t$.
  First, denote by $\blue{I} = \intcc{\lfloor \frac{t}{4} \rfloor + 1, t-1}$.
  In this notation, we claim that the following implication holds:
  \begin{equation}
  \cbr{T_{a^*}(t-1) < \frac t 2 - 1} \bigcap \del{\bigcap_{s\in I} \wbar{B_s}} ~~\bigsubseteq[1.7]~~
  \cbr{\exists a \in \cA'_\psi, s \in I \centerdot \wbar{B_s}, a_s = a, T_a(s-1) \geq \frac{t}{4K_\psi} - 1}.
  \label{eqn:implication-i}
  \end{equation}
  Indeed, if $\bigcap_{s\in I} \wbar{B_s}$ holds, then the chosen arm $a_s$ at time step $s$ must come from $\cA_\psi = \cA'_\psi \cup \cbr{a^*}$; the reason is as follows:
  \begin{enumerate}
    \item if $\Ex_s$, then $a_s = a^*$ is pulled;
    \item otherwise, $a_s$ is drawn from $\pi_s$ which is supported on $\cA_\psi$. 
  \end{enumerate}
  Throughout time interval $I$, we note that there are $\geq t - 1 - \lfloor \frac{t}{4} \rfloor \geq \frac{3}{4}t - 1$ time steps.
  Given the premise that $T_{a^*}(t-1) < \frac t 2 - 1$, the number of arm pulls of $a^*$ in $I$ must be $< \frac t 2 - 1$; this implies that
  the total number of arm pulls in $\cA'_\psi$ in $I$ must be greater than $(\frac{3t}{4} - 1) - (\frac t 2 - 1) \geq \frac t 4$.
  By pigeonhole's principle, there exists an arm $\blue{a_0} \in \cA'_\psi$ such that the number of arm pulls of $a_0$ in time span $I$ is at least $\frac{t}{4K_\psi}$ . Let $s$ be the last time step in $I$ when $a_0$ is pulled; therefore, we have $a_s = a_0$,   $T_{a_0}(s-1) \geq \frac{4}{4K_\psi} - 1$, and $\wbar{B_s}$ holding simultaneously, proving the above implication.
   
  Translating Equation~\eqref{eqn:implication-i} is equivalent to
  \[
  \cbr{ T_{a^*}(t-1) < \frac t 2 - 1 } \bigcap \del{\bigcap_{s\in I} \wbar{B_s}}
  ~~\bigsubseteq[1.7]~~
  \bigcup_{s  \in I} \cbr{\wbar{B_s}, \exists a \in \cA'_\psi \centerdot a_s = a, T_a(s-1) \geq \frac{t}{4K_\psi}-1}.
  \]
  
  Therefore, by the elementary fact that $\PP(U) \leq P(V) + P(\wbar{V} \cap U)$ and De Morgan's Law, we have
  \begin{align*}
  \PP \del{ T_{a^*}(t-1) < \frac t 2 - 1 }
  \leq
  \PP\del{ \bigcup_{s \in I} B_s }
  +
  \PP\del{ \bigcup_{s  \in I} \cbr{\wbar{B_s}, \exists a \in \cA'_\psi \centerdot a_s = a, T_a(s-1) \geq \frac{t}{4K_\psi} -1 } }.
  \end{align*}
  
  For the first term, by Lemma~\ref{lem:concentration-fstar-anytime}, we have
  \[
  \PP\del{ \bigcup_{s  \in I} B_s }
  \leq
  \PP\del{ \bigcup_{s  \geq \lfloor \frac{t}{4} \rfloor + 1} B_s }
  \leq
  \del{ \frac{1}{\lfloor \frac{t}{4} \rfloor + 1} }^2
  \leq \frac{16}{t^2}.
  \]
  For the second term, we have:
  \begin{align*}
  & \PP \del{ \bigcup_{s  \in I} \cbr{\wbar{B_s}, \exists a \in \cA'_\psi \centerdot a_s = a, T_a(s-1) \geq \frac{t}{4K_\psi} -  1}} \\
  & \leq
  \sum_{a \in \cA'_\psi} \PP\del{ \exists s \in I \centerdot \wbar{B_s}, a_s = a, T_a(s-1) \geq \frac{t}{4K_\psi} - 1} \\
  & \leq
  \sum_{a \in \cA'_\psi} \PP\del{ \exists s \in I \centerdot \wbar{\Ex_s}, a_s = a, T_a(s-1) \geq 8\fr{\beta_t}{\sig^2} \cdot \psi_a(\cF)  } \\
  & \leq
  \sum_{a \in \cA'_\psi} \PP\del{ \exists s \in I \centerdot \wbar{\Ex_s}, a_s = a, T_a(s-1) \geq 8 \fr{\beta_t}{\sig^2} \cdot \pi_{s,a} } \\
  & \leq
  K_\psi \cdot \abs{\cF} \cdot \frac{1}{zt^\alpha} = \frac{K_\psi}{t^2},
  \end{align*}
  where the first inequality is by union bound;
  the second inequality is from the definition of $\tau$: for all $t \geq \tau + 1$, $\frac{t}{4K_\psi} - 1 \geq  8  \fr{\beta_t}{\sig^2} (\max_a \psi_a(\cF)) \geq  8\fr{\beta_t}{\sig^2} \cdot \psi_a(\cF)$  and the fact that $\cA'_\psi$ does not contain $a^*$;
  the third inequality is from the observation that $\psi_a(\cF) \geq \psi_a(\barf_s) \geq \pi_{s,a}$;
  the fourth inequality is from Lemma~\ref{lem:tracking}\ref{item:pi-elim}; the last inequality is by algebra.
  
  To summarize,
  \[
  \PP \del{ T_{a^*}(t-1) < \frac t 2 - 1 } \leq \frac{K_\psi + 16}{t^2}.\
  \]
  Summing over all $t$'s, we get that
  \begin{align*}
  \Delta_{\max} \sum_{t=\tau+1}^n \PP \del{ T_{a^*}(t-1) < \frac t 2 - 1 }
  ~\leq  \Delta_{\max} \sum_{t=1}^\infty \frac{K_\psi + 16}{t^2}
  ~\leq 2(K_\psi + 16) \Delta_{\max} ~.
  \end{align*}
  
  \paragraph{Putting all together.} Combining the bounds on $\zoa, \zob, \zoc$, we have
  \[ \zo \leq O\del{ \del{ \sum_a \Delta_a \psi_a(\cF) }\cd \del{\ln\abs{\cF} + \ln(Q_1)} + K_\psi \Delta_{\max} }. \]
  Applying the definition of $P_3$ concludes the proof.
\end{proof}

\subsection{Miscellaneous lemmas}
\begin{lem}\label{lem:invert2}
  Let $A,B>0$. Then,
  $
  t < A + B \ln(t) \implies t < 2A + 2B\ln(B) 
  $
\end{lem}
\begin{proof}
  We use $\ln(t) \le t - 1$:
  \begin{align*}
  t 
  &< A + B\ln(t)
  \\ &= A + B\ln(\fr{t}{2B}\cd 2B)
  \\ &\le A + B\del{\fr{t}{2B} +\ln(\fr{2}{e}B)}
  \\ &= A + \fr{t}{2}+B\ln(B)
  \\\implies t &<  2A + 2B\ln(B)~.
  \end{align*}
\end{proof}

\begin{lem}\label{lem:tau}
  Suppose $\tau$ is defined as in Equation~\eqref{eqn:tau}. Then,
  \begin{align}
  \log(\tau) &= O(\ln (Q_1) + \ln(\ln(|\cF|))  )~. \label{eq:log-tau}
  \end{align}
  where $Q_1$ is defined in \Cref{thm:main-appendix}.
\end{lem}

\begin{proof}
  Let $\blue{R} := \max_a \psi_a(\cF)$.
  We have the following three equations from the definition of $\tau$.
  \begin{align*}
    \fr{\tau}{2} < 1+\fr{4}{\Lam_{\min}^2}\cd 4\ln(|\cF|\tau^2)
    &\implies \tau < 2 + \fr{64}{\Lam_{\min}^2} \ln|\cF| + \fr{64}{\Lam_{\min}^2} \ln(\tau)
    \\
    \fr{\tau}{4K_\psi} < 1+ 8\cd 4 R\cd \ln(|\cF|\tau^2)
    &\implies \tau < 4K_\psi  + 256 K_\psi R  \ln|\cF| + 256 K_\psi R  \ln(\tau)
~.
  \end{align*}
  By the definition of $\tau$, we have
  \begin{align*}
    \tau &\le \max\cbr{ 2 + \fr{64}{\Lam_{\min}^2} \ln|\cF| + \fr{64}{\Lam_{\min}^2} \ln(\tau) ,~~ 4K_\psi  + 256 K_\psi R  \ln|\cF| + 256 K_\psi R  \ln(\tau)}
    \\&\le \max\cbr{ 2 + \fr{64}{\Lam_{\min}^2} \ln|\cF| + \fr{64}{\Lam_{\min}^2} \ln(\tau) ,~~ 8K_\psi,~~  512 K_\psi R  \ln|\cF| + 512 K_\psi R  \ln(\tau) }
  \end{align*}
  where the second inequality is by $a + b \le \max\cbr{2a,2b}$.
  We can compactly write down $\tau \le \max\{8K_\psi,~~ A+B\ln(\tau)\} $ with 
  \begin{align*}
    A &= \max\cbr{ 2 + \fr{64}{\Lam_{\min}^2} \ln|\cF| ,~~ 512 K_\psi R  \ln|\cF|}
    \\B&=\max\cbr{ \fr{64}{\Lam_{\min}^2}  , ~~   512 K_\psi R } ~.
  \end{align*}  
  Then, by \Cref{lem:invert2}, we have
  \begin{align*}
    \tau < \max\{8K_\psi, A + B \ln(B) \} \le 8K_\psi +  A + B \ln(B) 
  \end{align*}
  Let $\blue{\xi} = \Lam_{\min}^{-2} + K_\psi R$.  
  Because $A = \Theta(1 + \xi \ln(|\cF|))$ and $B = \Theta(\xi)$,
  \begin{align}
  \tau &= O\lt(K_\psi + \xi\cd \ln(|\cF|) + \xi\ln(\xi)\rt) \notag
  \\\text{and}~~~~ \ln(\tau) &= O(\ln(\Lam_{\min}^{-2} + K_\psi (1+ R)) + \ln(\ln(|\cF|))  )~. 
  \end{align}
  
\end{proof}

\section{Lower bound}
\label{sec:lb-proof}

For our lower bound, we consider the following instance that resembles $\cH^+$ from Figure~\ref{fig:staircase}.
\begin{example} \label{ex:conflict}
  Let $\eps, \Lam>0$ and $r>1$. Suppose $\eps$ is small enough to ensure that 
  $f_2$ has the only informative arm of 3 and $f_3$ has the only informative arm of 4.
  \begin{itemize}
    \item $f_1 = (1, 1+\eps, 0, 0)$
    \item $f_2 = (1,  1-\eps, \Lam, 0)$
    \item $f_3 = (1,  1-\eps, \Lam, r\Lam)$
  \end{itemize}
\end{example}

Let us denote by $\EE_i T_j(n)$ the expected number of pulls of arm $j$  under the instance $f_i$.
We state our lower bound result in the following theorem.

\begin{thm}
\label{thm:lb-appendix}
Consider \Cref{ex:conflict}.
  Assume the Gaussian noise model with $\sig^2 = 1$.
  Suppose a bandit algorithm has $ \EE_1 T_1(n) = O(n^u)$ for some $u \in \lbrack 0, 1 \rparen$, then, for sufficiently large $n$,
\[
\EE_3 T_2(n) \vee \EE_3 T_3(n) \vee \EE_2 T_4(n) 
\geq \fr{6}{5} \fr{1}{r^2\Lam^2} \ln\del{1 + \fr{(1-u)\ln(n)}{48}}
 = \Omega(\ln(1 + (1-u) \ln n)).
\]
\end{thm}
At first sight, intuition on why the statement is true is not obvious since, assuming $\EE_3 T_2(n) \vee \EE_3 T_3(n)$ is $O(\ln(\ln((1-u)n))$, somehow the fact that arm 1 is not pulled sufficiently under $f_1$ implies a lower bound on arm 4 under $f_2$.
To explain this, let us consider the contraposition: If, $\EE_2 T_4(n) < O(\ln(\ln((1-u)n))$ for large enough $n$, then $T_1(n) = \Omega(n^u)$.
What happens in a nutshell is as follows.
The fact that $\EE_2 T_4(n)$ is not sufficient means that the algorithm cannot distinguish between $f_2$ and $f_3$ with probability approaching to 1.
Thus, roughly speaking, the behavior of the algorithm under $f_2$ and $f_3$ must be very similar for most times.
Together with the assumption on $\EE_3 T_2(n)$ and $\EE_3 T_3(n)$, the algorithms does not collect sufficient number of samples from arm 2, 3, and 4 under both $f_2$ and $f_3$.
The implication is that $(i)$ the remaining pulls all go to arm 1 and $(ii)$ the algorithm cannot distinguish between $f_1$ from $\{f_2,f_3\}$ either for most times, so it collects a lot of samples from arm 1 even under $f_1$.
The specific reason why $\ln(\ln(n))$ appears is quite technical.

This has an implication on forced sampling, as discussed in Section~\ref{sec:lb}.
That is, the contraposition of \Cref{thm:lb-appendix} implies that na\"ively pulling $\Theta(\ln(\ln(n)))$ for each arm may lead to regret that is arbitrarily close to being linear, let alone being uniformly good or being around the asymptotic optimality!

\begin{proof}
  Let $Y_1,\ldots,Y_4 >0$ be some constants such that $Y_1 + Y_2 + Y_3 + Y_4 = n$, to be tuned later. 
  Define $\blue{Y_{i:j}}$ as $Y_i,Y_{i+1},\ldots,Y_j$ for convenience.
  Denote by $\EE_i$ and $\PP_i$ the expectation and probability under $f_i$, respectively.
  Using divergence decomposition and Bretagnolle–Huber inequality~\cite[Lemma 15.1 and Theorem 14.2, respectively]{lattimore20bandit}, we have
  \begin{align*}
    & \fr12\exp\del{-\EE_2\sbr{\sum_{t=1}^n \KL(f_2(a_t), f_3(a_t))}}
    \\&= \fr12\exp(-\fr{r^2\Lam^2}{2} \EE_2 T_4(n))
    \\&\le \PP_2(T_4(n) \ge Y_4) + \PP_3(T_4(n) < Y_4)
    \\&= \PP_2(T_4(n) \ge Y_4) + \PP_3(T_4(n) < Y_4, T_3(n) \ge Y_3) 
    \\&\qquad\qquad\qquad~~~~~~\ + \PP_3(T_4(n) < Y_4, T_3(n) < Y_3, T_2(n) \ge Y_2) 
    \\&\qquad\qquad\qquad~~~~~~\ + \PP_3(T_4(n) < Y_4, T_3(n) < Y_3, T_2(n) < Y_2, T_1(n) \ge Y_1)
  \end{align*}
  
  On the other hand,
  \begin{align*}
    &\EE_2 T_4(n) + \EE_3 T_3(n) + \EE_3 T_2(n)
    \\&\ge Y_4 \PP_2(T_4(n) \ge Y_4) + Y_3 \PP_3(T_3(n) \ge Y_3) + Y_2 \PP_3(T_2(n)\ge Y_2)
    \\&\ge \min\{Y_2,Y_3,Y_4\} \cd \lt(\PP_2(T_4(n) \ge Y_4) + \PP_3(T_3(n) \ge Y_3) + \PP_3(T_2(n)\ge Y_2) \rt)
  \end{align*}
  Together,
  \begin{align}\label{eq:lbplan}
    \begin{aligned}
      \fr12 \exp(-\fr{r^2\Lam^2}{2} \EE_2 T_4(n))
      &\le \fr{1}{\min\{Y_{2:4}\}} \lt(   \EE_2 T_4(n) + \EE_3 T_3(n) + \EE_3 T_2(n) \rt)
      \\&\qquad+ \ubrace{\PP_3(T_4(n) < Y_4, T_3(n) < Y_3, T_2(n) < Y_2, T_1(n) \ge Y_1)}{=:\blue{Q_n}}
    \end{aligned}
  \end{align}
  The main effort is spent on bounding $Q_n$.
  
  Recall that reward distribution is Gaussian with variance $\sig^2 = 1$.
  Let $\blue{p_f(r_s\mid a_s)}$ be the pdf of the reward distribution under $f^* = f$ when arm $a_s$ is pulled at time $s$.
  
  We have the following anytime inequality. Denote by $\PP_{f^*}(\cd)$ be the probability of an event when $f^*$ is the ground truth.
  The following lemma states that under $f^*$ the empirical KL-divergence is not too far from the  KL-divergence (that is controlled by expected arm pulls made by the algorithm) with high probability.
  \begin{lem} \label{lem:KL}
    For every $\rho>0$,
    \begin{align*}
      \PP_{f^*} \lt(\blue{B(f^*,f)} := \cbr{\exists t\ge1, ~  \sum_{s=1}^t \ln\fr{ p_{f^*}(r_s \mid  a_s)}{ p_{f}(r_s \mid  a_s)} \ge (1+\rho) \sum_{s=1}^t \KL(f^*(a_s),f(a_s)) + \fr{1}{\rho}\ln(\dt^{-1}) } \rt)\le \dt
    \end{align*}
  \end{lem}
  \begin{proof}
    Using Equation~\eqref{eqn:sq-ub} of~\Cref{lem:sq-bounds}, we have
    \begin{align*}
      \PP_{f^*}\lt(\exists t\ge1, ~~  \sum_{s}^t M_s(f) - (1+2\sig^2\lam) \sum_s^t \EE_s[M_s(f)] \ge \fr1\lam \ln(\dt^{-1}) \rt)&\le \dt
    \end{align*}
    Notice that $M_s = 2\sig^2 \ln\fr{ p_{f^*}(r_s \mid a_s)}{ p_{f}(r_s \mid a_s)}$ and $\EE_s[M_s(f)] = 2\sig^2 \KL(f^*(a_s),f(a_s))$.
    Then,
    \begin{align*}
      \PP_{f^*}\lt(\exists t, ~  \sum_s^t \ln\fr{ p_{f^*}(r_s \mid a_s)}{ p_{f}(r_s \mid a_s)} - (1+2\sig^2\lam) \sum_{s=1}^t  \KL(f^*(a_s),f(a_s)) \ge \fr1{2\sig^2 \lam}\ln(\dt^{-1}) \rt)\le \dt ~ .
    \end{align*}
    A simple change of variable concludes the proof.
  \end{proof}

Let us define $\blue{\beta_4} = \fr{2}{r^2\Lam^2}$, $\blue{Y} = \fr{w}{3} \beta_4 \ln(n)$, and $\blue{A_n} = \cbr{T_2(n), T_3(n), T_4(n) \le  Y}$ for some $\blue{w} \in (0,1)$ that we tune later.
Assume that $\EE_1 T_1(n) = O( n^u)$ for some $\blue{u}\in \lbrack 0, 1 \rparen$.
Recall that we want to upper bound $\PP_3(A_n)$ from~\eqref{eq:lbplan} for which we plan to use the change of measure argument.
Specifically, we observe that, for large enough $n$,
\begin{align*}
  \PP_1(A_n) \le \PP_1\del{T_1(n) \ge n - w\beta_4 \ln(n)} 
  \le \fr{\EE_1 T_1(n)}{n - w\beta_4 \ln(n)}
  \le \fr{\EE_1 T_1(n)}{n/2}
  ~\sr{(a)}{\le} c_1 n^{u-1}
\end{align*}
for some constant $c_1>0$ where $(a)$ is by our assumption on $\EE_1 T_1(n)$.

Because we have set $\eps$ to be small enough, we have  ${\beta_4} = \fr{2}{r^2\Lam^2} \le \fr{2}{\Lam^2} ~\wedge~ \fr{2}{4\eps^2}$.
Then, under $A_n$, we have
\begin{align}
  \sum_{s=1}^{t} \KL(f_3(a_s), f_1(a_s))
  & = T_2(n) \cd \fr{4\eps^2}{2} + T_3(n) \cd \fr{\Lam^2}{2} + T_4(n) \cd\fr{r^2\Lam^2}{2} \nonumber \\
  & \leq 3 Y \max\del{ \frac{4\epsilon^2}{2}, \frac{\Lam^2}{2}, \frac{r^2 \Lam^2}{2} }
  \le w \ln(n) \label{eq:lb-klbound} 
\end{align}
We now lower bound $\PP_1(A_n)$. Recall the definition of $B(\cdot,\cd)$ from Lemma~\ref{lem:KL}.
\begin{align*}
  \PP_1\del{ A_n }
  &\ge \PP_1\del{ A_n, \wbar{B(f_3,f_1)} }
  \\&= \EE_3 \sbr {\one\{A_n, \wbar{B(f_3,f_1)}\} \prod_{t=1}^n \fr{p_1(r_t \mid  a_t)}{p_3(r_t \mid  a_t)} }
  \\&= \EE_3 \sbr {\one\{A_n, \wbar{B(f_3,f_1)}\} \exp\del{-\sum_{t=1}^n \ln\fr{p_3(r_t \mid  a_t)}{p_1(r_t \mid  a_t)}} }
  \\&\ge \EE_3 \sbr {\one\{A_n, \wbar{B(f_3,f_1)}\} \exp\del{-\del{
        (1+\rho) \sum_{t=1}^{n} \KL(f_3(a_t),f_1(a_t)) + \fr{1}\rho \ln(\dt^{-1})
  }} }
  \\&\sr{\eqref{eq:lb-klbound}}{\ge} \EE_3 \sbr {\one\{A_n, \wbar{B(f_3,f_1)}\} \exp\del{-\del{
        (1+\rho)\cd w\ln(n) + \fr{1}{\rho}\ln(\delta^{-1})
  }} }
  \\&= \PP_3\del{A_n, \wbar{B(f_3,f_1)}} \cd \del{1/n}^{(1+\rho)w} \cd \dt^{\fr{1}{\rho}}
  \\&\ge \del{\PP_3\del{A_n} - \PP_3(B(f_3,f_1)) }\cd \del{1/n}^{(1+\rho) w} \cd \dt^{\fr{1}{\rho}} 
  \tag*{(\bec $\PP(A) \le \PP(A,\wbar B) + \PP(B)$)}
  \\&\ge \del{\PP_3\del{A_n} - \dt}\cd \del{1/n}^{(1+\rho) w} \cd \dt^{\fr{1}{\rho}} 
  \tag*{(\bec \Cref{lem:KL})}
\end{align*}
Combining the lower and upper bound on $\PP_1(A_n)$ above,
\begin{align*}
  \del{\PP_3\del{A_n} - \dt}\cd \del{1/n}^{(1+\rho) w} \cd \dt^{\fr{1}{\rho}} &\le c_1 n^{u-1}
  \\ \implies 
  \PP_3\del{A_n}
  &\le \delta + c_1 n^{u-1} \cd  \del{1/n}^{-(1+\rho) w} \cd \dt^{\textstyle-\fr{1}{\rho}}
  \\&\le (1/n)^{q} + c_1 (1/n)^{ 1-u - (1+\rho)w - \fr{q}{\rho}} 
  \tag*{(set $\blue{\dt}=(1/n)^q$)}
\end{align*}
By setting $\blue{w} = \blue{q} = \fr{1-u}{4}$ and $\rho = 1$, we have $1-u - (1+\rho)w - \fr{q}{\rho} = \fr{1-u}{4}$.
Then,
With this choice, we have
\begin{align*}
  \PP_3(A_n) \le (1/n)^{\fr{1-u}{4}} + c_1 \cd (1/n)^{\fr{1-u}{4}}
\end{align*}
Using our choice of $Y_{2:4} = \fr{\beta_4}{6}\ln(n)$, we go back to where we began:

\begin{align*}
  &\fr12 \exp(-\frac{r^2\Lam^2}{2} \EE_2 T_4(n))
  \\&\le \fr{1}{\min\{Y_{2:4}\}} \lt( \EE_2 T_4(n) + \EE_3 T_3(n) + \EE_3 T_2(n) \rt)        + Q_n
  \\&\le \fr{1}{(w/3)\beta_4\ln(n)} \cd 3 \cdot \del{\EE_2 T_4(n) \vee \EE_3 T_3(n) \vee \EE_3 T_2(n)} + (1/n)^{\fr{1-u}{4}} + c_1 \cd (1/n)^{\fr{1-u}{4}}
\end{align*}

Denote by $R = \EE_3 T_3(n) \vee \EE_3 T_2(n) \vee \EE_3 T_4(n)$. 
One can see that, if $R$ is uniformly bounded w.r.t. $n$, we get a contradiction because the LHS is bounded below but the RHS gets smaller with $n$.
Therefore, $R$ must grow indefinitely over time.
This implies that, for large enough $n$, we have $C \le R$ and $(1/n)^{\fr{1-u}{4}} + c_1 \cd (1/n)^{\fr{1-u}{4}} \le \fr{R}{(w/3)\beta_4\ln(n)}$.
Then,
\begin{align*}
  \fr12 \exp\del{- \frac{r^2\Lam^2}{2} R} \le \fr{4 R}{(w/3)\beta_4 \ln(n)} 
  = \fr{(12/w) R}{\fr{2}{r^2\Lam^2} \ln(n)}
\end{align*}
It remains to solve the above for $R$.
We do so by inverting the Lambert function.
Let $\blue{Z} = \frac{r^2 \Lam^2 R}{2}$. 
Then,
\begin{align*}
  \exp(-Z) &\le \fr{(12/w) Z}{\ln(n)}
  \\ \fr{w}{12} \ln(n) &\le Z \exp(Z) =: \blue{X}
\end{align*}
We like to find $Z(X)$ that satisfies $Z(X)\exp(Z(X)) = X$.
Using \citet[Lemma 17]{orabona16coin}, we have $\fr35 \ln(X+1) \le Z(X) \le \ln(X+1)$.
Therefore,
\begin{align*}
  Z(X) \ge \fr35 \ln(X+1) ~\ge \fr{3}{5} \ln\del{\fr{w\ln(n)}{12} + 1}
\end{align*}
Substituting $Z(X)$ and $w$ with their definitions concludes the proof.
\end{proof}

\putbib[library-shared]
\end{bibunit}

\bibliography{library-shared} 
\end{document}